\documentclass[11pt]{article}
\usepackage{microtype}
\usepackage{graphicx}
\usepackage{caption}
\usepackage{subcaption}
\usepackage{comment}
\usepackage{booktabs}
\usepackage{hyperref}

\usepackage{natbib}

\usepackage[utf8]{inputenc}
\usepackage{amsmath}
\usepackage{amsfonts}
\usepackage{amsthm}
\usepackage{amssymb}
\usepackage[algoruled,vlined,nofillcomment]{algorithm2e}

\newcommand{\cN}{\mathcal{N}}
\renewcommand{\Pr}{\mathbb{P}}
\newcommand{\R}{\mathbb{R}}
\newcommand{\N}{\mathbb{N}}
\newcommand{\X}{\mathbb{X}}
\newcommand{\Y}{\mathbb{Y}}
\newcommand{\E}{\mathbb{E}}

\newcommand{\norm}[1]{\| #1 \|}

\newcommand{\cB}{\mathcal{B}}
\newcommand{\one}{\mathbf{1}}

\newcommand{\argmin}{\mathop{\mathrm{argmin}}}

\newtheorem{theorem}{Theorem}
\newtheorem{lemma}[theorem]{Lemma}

\usepackage{authblk}

\usepackage{fullpage}
\title{Improving the Gaussian Mechanism for Differential Privacy: Analytical Calibration and Optimal Denoising\footnote{To appear at the 35th International Conference on Machine Learning (ICML), 2018}}
\date{}
\author[1]{Borja Balle\footnote{Corresponding e-mail: \url{pigem@amazon.co.uk}}}
\affil[1]{Amazon Research, Cambridge, UK}
\author[2]{Yu-Xiang Wang}
\affil[2]{Amazon Web Services, Palo Alto, USA}

\begin{document}

\maketitle

\begin{abstract}
The Gaussian mechanism is an essential building block used in multitude of differentially private data analysis algorithms. In this paper we revisit the Gaussian mechanism and show that the original analysis has several important limitations. Our analysis reveals that the variance formula for the original mechanism is far from tight in the high privacy regime ($\varepsilon \to 0$) and it cannot be extended to the low privacy regime ($\varepsilon \to \infty$). We address these limitations by developing an optimal Gaussian mechanism whose variance is calibrated directly using the Gaussian cumulative density function instead of a tail bound approximation. We also propose to equip the Gaussian mechanism with a post-processing step based on adaptive estimation techniques by leveraging that the distribution of the perturbation is known. Our experiments show that analytical calibration removes at least a third of the variance of the noise compared to the classical Gaussian mechanism, and that denoising dramatically improves the accuracy of the Gaussian mechanism in the high-dimensional regime.
\end{abstract}

\section{Introduction}

Output perturbation is a cornerstone of mechanism design in differential privacy (DP). Well-known mechanisms in this class are the Laplace and Gaussian mechanisms \cite{dwork2006calibrating,dwork2014algorithmic}. More complex mechanisms are often obtained by composing multiple applications of these basic output perturbation mechanisms. For example, the Laplace mechanism is the basic building block of the sparse vector mechanism \cite{dwork2009complexity}, and the Gaussian mechanism is the building block of private empirical risk minimization algorithms based on stochastic gradient descent \cite{bassily2014private}. Analysing the privacy of such complex mechanisms turns out to be a delicate and error-prone task \cite{lyu2017understanding}. In particular, obtaining tight privacy analyses leading to optimal utility is one of the main challenges in the design of advanced DP mechanisms. An alternative to tight \emph{a-priori} analyses is to equip complex mechanisms with \emph{algorithmic} noise calibration and accounting methods. These methods use numerical computations to, e.g.\ calibrate perturbations and compute cumulative privacy losses at run time, without relying on hand-crafted worst-case bounds. For example, recent works have proposed methods to account for the privacy loss under compositions occurring in complex mechanisms \cite{rogers2016privacy,abadi2016deep}.

In this work we revisit the Gaussian mechanism and develop two ideas to improve the utility of output perturbation DP mechanisms based on Gaussian noise. The first improvement is an algorithmic noise calibration strategy that uses numerical evaluations of the Gaussian cumulative density function (CDF) to obtain the optimal variance to achieve DP using Gaussian perturbation. The analysis and the resulting algorithm are provided in Section~\ref{sec:aGM}.
In order to motivate the need for a numerical approach to calibrate the noise of a DP Gaussian perturbation mechanism, we start with an analysis of the main limitations of the classical Gaussian mechanism in Section~\ref{sec:limitations}.
A numerical evaluation provided in Section~\ref{sec:expaGM} showcases the advantages of our optimal calibration procedure.

The second improvement equips the Gaussian perturbation mechanism with a post-processing step which denoises the output using adaptive estimation techniques from the statistics literature. Since DP is preserved by post-processing and the distribution of the perturbation added to the desired outcome is known, this allows a mechanism to achieve the desired privacy guarantee while increasing the accuracy of the released value. The relevant denoising estimators and their utility guarantees are discussed in Section~\ref{sec:denoise}. Results presented in this section are not new: they are the product of a century's worth of research in statistical estimation. Our contribution is to compile relevant results scattered throughout the literature in a single place and showcase their practical impact in synthetic (Section~\ref{sec:expmean}) and real (Section~\ref{sec:exptaxi}) datasets, thus providing useful pointers and guidelines for practitioners.

\section{Limitations of the Classical Gaussian Mechanism}\label{sec:limitations}

Let $\X$ be an input space equipped with a symmetric neighbouring relation $x \simeq x'$. Let $\varepsilon \geq 0$ and $\delta \in [0,1]$ be two privacy parameters. A $\Y$-valued randomized algorithm $M : \X \to \Y$ is $(\varepsilon,\delta)$-DP \cite{dwork2006calibrating} if for every pair of neighbouring inputs $x \simeq x'$ and every possible (measurable) output set $E \subseteq \Y$ the following inequality holds:
\begin{align}\label{eqn:DP}
\Pr[M(x) \in E] \leq e^{\varepsilon} \Pr[M(x') \in E] + \delta \enspace.
\end{align}
The definition of DP captures the intuition that a computation on private data will not reveal sensitive information about individuals in a dataset if removing or replacing an individual in the dataset has a negligible effect in the output distribution.

In this paper we focus on the family of so-called output perturbation DP mechanisms. 
An output perturbation mechanism $M$ for a deterministic vector-valued computation $f : \X \to \R^d$ is obtained by computing the function $f$ on the input data $x$ and then adding random noise sampled from a random variable $Z$ to the output.
The amount of noise required to ensure the mechanism $M(x) = f(x) + Z$ satisfies a given privacy guarantee typically depends on how sensitive the function $f$ is to changes in the input and the specific distribution chosen for $Z$.
The Gaussian mechanism gives a way to calibrate a zero mean isotropic Gaussian perturbation $Z \sim \cN(0,\sigma^2 I)$ to the global $L_2$ sensitivity $\Delta = \sup_{x \simeq x'} \norm{f(x) - f(x')}$ of $f$ as follows.

\begin{theorem}[Classical Gaussian Mechanism]\label{thm:cGM}
For any $\varepsilon, \delta \in (0,1)$, the Gaussian output perturbation mechanism with $\sigma = \Delta \sqrt{2 \log(1.25/\delta)}/\varepsilon$ is $(\varepsilon,\delta)$-DP.
\end{theorem}

A natural question one can ask about this result is whether this value of $\sigma$ provides the minimal amount of noise required to obtain $(\varepsilon,\delta)$-DP with Gaussian perturbations. Another natural question is what happens in the case $\varepsilon \geq 1$. This section addresses both these questions. First we show that the value of $\sigma$ given in Theorem~\ref{thm:cGM} is suboptimal in the high privacy regime $\varepsilon \to 0$. Then we show that this problem is in fact inherent to the usual proof strategy used to analyze the Gaussian mechanism. We conclude the section by showing that for large values of $\varepsilon$ the standard deviation of a Gaussian perturbation that provides $(\varepsilon,\delta$)-DP must scale like $\Omega(1/\sqrt{\varepsilon})$. This implies that the scaling $\Theta(1/\varepsilon)$ provided by the classical Gaussian mechanism in the range $\varepsilon \in (0,1)$ cannot be extended beyond any bounded interval.

\subsection{Limitations in the High Privacy Regime}\label{sec:highprivacy}

To illustrate the sub-optimality of the classical Gaussian mechanism in the regime $\varepsilon \to 0$ we start by showing it is possible to achieve $(0,\delta)$-DP using Gaussian perturbations. This clearly falls outside the capabilities of the classical Gaussian mechanism, since the standard deviation $\sigma = \Theta(1/\varepsilon)$ provided by Theorem~\ref{thm:cGM} grows to infinity as $\varepsilon \to 0$.

\begin{theorem}\label{thm:TV}
A Gaussian output perturbation mechanism with $\sigma = \Delta / 2 \delta$ is $(0,\delta)$-DP\footnote{Proofs for all results given in the paper are presented in Appendix~\ref{app:proofs}.}.
\end{theorem}

Previous analyses of the Gaussian mechanism are based on a simple sufficient condition for DP in terms of the privacy loss random variable \cite{dwork2014algorithmic}.
The next section explains why the usual analysis of the Gaussian mechanism cannot yield tight bounds for the regime $\varepsilon \to 0$. 
This shows that our example is not a corner case, but a fundamental limitation of trying to establish $(\varepsilon,\delta)$-DP through said sufficient condition.

\subsection{Limitations of Privacy Loss Analyses}\label{sec:prooflimitations}

Given a vector-valued mechanism $M$ let $p_{M(x)}(y)$ denote the density of the random variable $Y = M(x)$. The privacy loss function of $M$ on a pair of neighbouring inputs $x \simeq x'$ is defined as
\begin{align*}
\ell_{M,x,x'}(y) = \log \left(\frac{p_{M(x)}(y)}{p_{M(x')}(y)}\right) \enspace.
\end{align*}
The privacy loss random variable $L_{M,x,x'} = \ell_{M,x,x'}(Y)$ is the transformation of the output random variable $Y = M(x)$ by the function $\ell_{M,x,x'}$. For the particular case of a Gaussian mechanism $M(x) = f(x) + Z$ with $Z \sim \cN(0,\sigma^2 I)$ it is well-known that the privacy loss random variable is also Gaussian \cite{dwork2016concentrated}.

\begin{lemma}\label{lem:privlossGM}
The privacy loss $L_{M,x,x'}$ of a Gaussian output perturbation mechanism follows a distribution $\cN(\eta, 2 \eta)$ with
$\eta = D^2 / 2 \sigma^2$, where $D = \norm{f(x) - f(x')}$.
\end{lemma}

The privacy analysis of the classical Gaussian mechanism relies on the following sufficient condition: a mechanism $M$ is $(\varepsilon,\delta)$-DP if the privacy loss $L_{M,x,x'}$ satisfies
\begin{align}\label{eqn:suff}
\forall x \simeq x' : \Pr[L_{M,x,x'} \geq \varepsilon] \leq \delta \enspace.
\end{align}
Since Lemma~\ref{lem:privlossGM} shows the privacy loss $L_{M,x,x'}$ of the Gaussian mechanism is a Gaussian random variable with mean $\norm{f(x) - f(x')}^2 / 2 \sigma^2$, we have $\Pr[L_{M,x,x'} > 0] \geq 1/2$ for any pair of datasets with $f(x) \neq f(x')$. This observation shows that in general it is not possible to use this sufficient condition for $(\varepsilon,\delta)$-DP to prove that the Gaussian mechanism achieves $(0,\delta)$-DP for any $\delta < 1/2$. In other words, the sufficient condition is not necessary in the regime $\varepsilon \to 0$. We conclude that an alternative analysis is required in order to improve the dependence on $\varepsilon$ in the Gaussian mechanism.

\subsection{Limitations in the Low Privacy Regime}\label{sec:lowprivacy}

The last question we address in this section is whether the order of magnitude $\sigma = \Theta(1/\varepsilon)$ given by Theorem~\ref{thm:cGM} for $\varepsilon \leq 1$ can be extended to privacy parameters of the form $\varepsilon > 1$. We show this is not the case by providing the following lower bound.

\begin{theorem}\label{thm:lowerb}
Let $f : \X \to \R^d$ have global $L_2$ sensitivity $\Delta$. Suppose $\varepsilon > 0$ and $0 < \delta < 1/2 - e^{-3 \varepsilon}/\sqrt{4 \pi \varepsilon}$. If the mechanism $M(x) = f(x) + Z$ with $Z \sim \cN(0,\sigma^2 I)$ is $(\varepsilon,\delta)$-DP, then $\sigma \geq \Delta / \sqrt{2 \varepsilon}$.
\end{theorem}

Note that as $\varepsilon \to \infty$ the upper bound on $\delta$ in Theorem~\ref{thm:lowerb} converges to $1/2$. Thus, as $\varepsilon$ increases the range of $\delta$'s requiring noise of the order $\Omega(1/\sqrt{\varepsilon})$ increases to include all parameters of practical interest. This shows that the rate $\sigma = \Theta(1/\varepsilon)$ provided by the classical Gaussian mechanism cannot be extended beyond the interval $\varepsilon \in (0,1)$. Note this provides an interesting contrast with the Laplace mechanism, which can achieve $\varepsilon$-DP with standard deviation $\Theta(1/\varepsilon)$ in the low privacy regime.

\section{The Analytic Gaussian Mechanism}\label{sec:aGM}

The limitations of the classical Gaussian mechanism described in the previous section suggest there is room for improvement in the calibration of the variance of a Gaussian perturbation to the corresponding global $L_2$ sensitivity.
Here we present a method for optimal noise calibration for Gaussian perturbations that we call \emph{analytic Gaussian mechanism}.
To do so we must address the two sources of slack in the classical analysis: the sufficient condition \eqref{eqn:suff} used to reduce the analysis to finding an upper bound for $\Pr[\cN(\eta,2\eta) > \varepsilon]$, and the use of a Gaussian tail approximation to obtain such upper bound.
We address the first source of slack by showing that the sufficient condition in terms of the privacy loss random variable comes from a relaxation of a necessary and sufficient condition involving two privacy loss random variables. When specialized to the Gaussian mechanism, this condition involves probabilities about Gaussian random variables, which instead of approximating by a tail bound we represent explicitly in terms of the CDF of the standard univariate Gaussian distribution:
\begin{align*}
\Phi(t) = \Pr[\cN(0,1) \leq t] = \frac{1}{\sqrt{2 \pi}} \int_{-\infty}^{t} e^{-y^2/2} dy \enspace.
\end{align*}
Using this point of view, we introduce a calibration strategy for Gaussian perturbations that requires solving a simple optimization problem involving $\Phi(t)$. We discuss how to solve this optimization at the end of this section.

The first step in our analysis is to provide a necessary and sufficient condition for differential privacy in terms of privacy loss random variables. This is captured by the following result.

\begin{theorem}\label{thm:necesuff}
A mechanism $M : \X \to \Y$ is $(\varepsilon,\delta)$-DP if and only if the following holds for every $x \simeq x'$:
\begin{align}\label{eqn:necesuff}
\Pr[L_{M,x,x'} \geq \varepsilon] - e^{\varepsilon} \Pr[L_{M,x',x} \leq - \varepsilon] \leq \delta \enspace.
\end{align}
\end{theorem}

Note that Theorem~\ref{thm:necesuff} immediately implies the sufficient condition given in \eqref{eqn:suff} through the inequality
\begin{align*}
\Pr[L_{M,x,x'} \geq \varepsilon] - e^{\varepsilon} \Pr[L_{M,x',x} \leq - \varepsilon] \leq
\Pr[L_{M,x,x'} \geq \varepsilon] \enspace.
\end{align*}
Now we can use Lemma~\ref{lem:privlossGM} to specialize \eqref{eqn:necesuff} for a Gaussian output perturbation mechanism. The relevant computations are packaged in the following result, where we express the probabilities in \eqref{eqn:necesuff} in terms of the Gaussian CDF $\Phi$.

\begin{lemma}\label{lem:LtoCDF}
Suppose $M(x) = f(x) + Z$ is a Gaussian output perturbation mechanism with $Z \sim \cN(0,\sigma^2 I)$. For any $x \simeq x'$ let $D = \norm{f(x) - f(x')}$. Then the following hold for any $\varepsilon \geq 0$:
\begin{align}
\Pr[L_{M,x,x'} \geq \varepsilon] &=
\Phi\left(\frac{D}{2 \sigma} - \frac{ \varepsilon \sigma}{D}\right) \label{eqn:L1} \enspace, \\
\Pr[L_{M,x',x} \leq - \varepsilon] &= \Phi\left(- \frac{D}{2 \sigma} - \frac{ \varepsilon \sigma}{D}\right) \label{eqn:L2} \enspace.
\end{align}
\end{lemma}

This result specializes the left hand side of \eqref{eqn:necesuff} in terms of the distance $D = \norm{f(x) - f(x')}$ between the output means on a pair of neighbouring datasets. To complete the derivation of our analytic Gaussian mechanism we need to ensure that \eqref{eqn:necesuff} is satisfied for \emph{every} pair $x \simeq x'$. The next lemma shows that this reduces to plugging the global $L_2$ sensitivity $\Delta$ in the place of $D$ in \eqref{eqn:L1} and \eqref{eqn:L2}.

\begin{lemma}\label{lem:monotone}
For any $\varepsilon \geq 0$, the function $h : \R_{\geq 0} \to \R$ defined as follows is monotonically increasing:
\begin{align*}
h(\eta) = \Pr[\cN(\eta,2\eta) \geq \varepsilon] - e^{\varepsilon} \Pr[\cN(\eta,2\eta) \leq - \varepsilon] \enspace.
\end{align*}
\end{lemma}

Now we are ready to state our main result, whose proof follows directly from Theorem~\ref{thm:necesuff}, Lemma~\ref{lem:monotone}, and equations \eqref{eqn:L1} and \eqref{eqn:L2}.

\begin{theorem}[Analytic Gaussian Mechanism]\label{thm:aGMDP}
Let $f : \X \to \R^d$ be a function with global $L_2$ sensitivity $\Delta$.
For any $\varepsilon \geq 0$ and $\delta \in [0,1]$, the Gaussian output perturbation mechanism $M(x) = f(x) + Z$ with $Z \sim \cN(0,\sigma^2 I)$ is $(\varepsilon,\delta)$-DP if and only if
\begin{align}\label{eqn:N}
\Phi\left(\frac{\Delta}{2 \sigma} - \frac{ \varepsilon \sigma}{\Delta}\right) - e^{\varepsilon} \Phi\left(- \frac{\Delta}{2 \sigma} - \frac{ \varepsilon \sigma}{\Delta}\right) \leq \delta \enspace.
\end{align}
\end{theorem}

This result shows that in order to obtain an $(\varepsilon,\delta)$-DP Gaussian output perturbation mechanism for a function $f$ with global $L_2$ sensitivity $\Delta$ it is enough to find a noise variance $\sigma^2$ satisfying \eqref{eqn:N}. One could now use upper and lower bounds for the tail of the Gaussian CDF to derive an analytic expression for a parameter $\sigma$ satisfying this constraint. However, this again leads to a suboptimal result due to the slack in these tail bounds in the non-asymptotic regime. Instead, we propose to find $\sigma$ using a numerical algorithm by leveraging the fact that the Gaussian CDF can be written as $\Phi(t) = (1 + \mathsf{erf}(t/\sqrt{2}))/2$, where $\mathsf{erf}$ is the standard error function. Efficient implementations of this function to very high accuracies are provided by most statistical and numerical software packages. However, this strategy requires some care in order to avoid numerical stability issues around the point where the expression $\Delta/2 \sigma - \varepsilon \sigma/\Delta$ in \eqref{eqn:N} changes sign. Thus, we further massage the left hand side \eqref{eqn:N} we obtain the implementation of the analytic Gaussian mechanism given in Algorithm~\ref{alg:newGM}. The correctness of this implementation is provided by the following result.

\begin{theorem}\label{thm:aGM}
Let $f$ be a function with global $L_2$ sensitivity $\Delta$. For any $\varepsilon > 0$ and $\delta \in (0,1)$, the mechanism described in Algorithm~\ref{alg:newGM} is $(\varepsilon,\delta)$-DP.
\end{theorem}

\begin{algorithm}[t]
\DontPrintSemicolon
\caption{Analytic Gaussian Mechanism}\label{alg:newGM}
\KwSty{Public Inputs:} $f$, $\Delta$, $\varepsilon$, $\delta$\;
\KwSty{Private Inputs:} $x$\;
Let $\delta_0 = \Phi(0) - e^{\varepsilon} \Phi(- \sqrt{2 \varepsilon})$\;
\eIf{$\delta \geq \delta_0$}{
Define $B_{\varepsilon}^+(v) = \Phi(\sqrt{\varepsilon v})
- e^{\varepsilon} \Phi(-\sqrt{\varepsilon (v + 2)})$\;
Compute $v^* = \sup \{ v \in \R_{\geq 0} : B_{\varepsilon}^+(v) \leq \delta \}$\;
Let $\alpha = \sqrt{1+ v^*/2} - \sqrt{v^*/2}$\;
}{
Define $B_{\varepsilon}^-(u) = \Phi(-\sqrt{\varepsilon u})
- e^{\varepsilon} \Phi(-\sqrt{\varepsilon (u + 2)})$\;
Compute $u^* = \inf \{ u \in \R_{\geq 0} : B_{\varepsilon}^-(u) \leq \delta \}$\;
Let $\alpha = \sqrt{1 + u^*/2} + \sqrt{u^*/2}$\;
}
Let $\sigma = \alpha \Delta / \sqrt{2 \varepsilon}$\;
Return $f(x) + \cN(0,\sigma^2 I)$
\end{algorithm}

Given a numerical oracle for computing $\Phi(t)$ based on the error function it is relatively straightforward to implement a solver for finding the values $v^*$ and $u^*$ needed in Algorithm~\ref{alg:newGM}. For example, using the fact that $B_{\varepsilon}^+(v)$ is monotonically increasing we see that computing $v^*$ is a root finding problem for which one can use Newton's method since the derivative of $\Phi(t)$ can be computed in closed form using Leibniz's rule. In practice we find that a simple scheme based on binary search initiated from an interval obtained by finding the smallest $k \in \N$ such that $B_{\varepsilon}^+(2^k) > \delta$ provides a very efficient and robust way to find $v^*$ up to arbitrary accuracies (the same applies to $u^*$).

\section{Optimal Denoising}\label{sec:denoise}
Can we improve the performance of analytical Gaussian mechanism even further? The answer is ``yes'' and ``no''. We can't because Algorithm~\ref{alg:newGM} is already the exact calibration of the Gaussian noise level to the given privacy budget. But if we consider the problem of designing the best differentially private procedure $M(x)$ that approximates $f(x)$, then there could still be room for improvement.

In this section, we consider a specific class of mechanisms that \emph{denoise} the output of a Gaussian mechanism. Let $\hat{y} \sim \cN(f(x),\sigma^2 I)$, we are interested in designing a post-processing function $g$ such that $\tilde{y} = g(\hat{y})$ is closer to $f(x)$ than $\hat{y}$. This class of mechanisms are of particular interest for differential privacy because  (1) since differential privacy is preserved by post-processing, releasing a function $\tilde{y} = g(\hat{y})$ of a differentially private output is again differentially private; (2) since information about $f$ and the distribution of the noise are publicly known, this information can be leveraged to design denoising functions.

This is a statistical estimation problem, where $f(x)$ is the underlying parameter and $\hat{y}$ is the data. Since in this case we are adding the noise ourselves, it is possible to use the classical statistical theory on Gaussian models \emph{as is} because the Gaussian assumption is now true by construction.
This is however an unusual estimation problem where all we observe is a single data point. Since $\hat{y}$ is the maximum likelihood estimator, if there is no additional information about $f(x)$, we cannot hope to improve the estimation error \emph{uniformly} over all $f(x) \in \R^d$. But there is still something we can do when we consider either of the following assumptions:
(A.1) $x$ is drawn from some underlying distribution, thus inducing some distribution on $f(x)$; or,
(A.2) $\|f(x)\|_p \leq B$ for some $p, B> 0$, where $\|\cdot\|_p$ is the $L_p$-norm (or pseudo-norm when $p < 1$).

\paragraph{Optimal Bayesian denoising.}
Assumption A.1 translates the problem of optimal denoising into a Bayesian estimation problem, where the underlying parameter $f(x)$ has a prior distribution, and the task is to find an estimator that attains the Bayes risk --- the minimum of the average estimation error integrated over a prior $\pi$, defined as
$$ R(\pi) = \min_{g:  \R^d \rightarrow \R^d } \E\big[ \E  [\norm{g(\hat{y}) - f(x)}^2 |  f(x)] \big] \enspace.$$ 
For square loss, the Bayes estimator is simply the posterior mean estimator, as the following theorem shows:
\begin{theorem}\label{thm:bayes}
		Let $x \sim  \pi$ and assume the induced distribution of $f(x)$ is square integrable.
		 Then the Bayes estimator $\tilde{y}_{\mathrm{Bayes}}$ is given by
	$$
	\tilde{y}_{\mathrm{Bayes}}  = \argmin_{g: \R^d \rightarrow \R^d}  \E\left[\|g(\hat{y}) - f(x)\|^2\right] =  \E[ f(x) | \hat{y}] \enspace.
	$$
\end{theorem}
The proof can be found in any standard statistics textbook \citep[see, e.g.,][]{lehmann2006theory}. 
One may ask what the corresponding MSE is and how much  it improves over the version without post-processing. The answer depends on the prior and the amount of noise added for differential privacy. When $f(x)\sim \cN(0, w^2I)$, the posterior mean estimator can be written analytically into
$\tilde{y}_{\mathrm{Bayes}}   =  (w^2 / (w^2+\sigma^2)) \hat{y}$,
and the corresponding Bayes risk is
$\E[\|\tilde{y}_{\mathrm{Bayes}} - f(x)\|^2]  = d w^2 \sigma^2 / (\sigma^2 + w^2)$.
In other word, we get a factor of $w^2/(w^2+\sigma^2)$ improvement over simply using $\hat{y}$.

In general, there is no analytical form for the posterior mean, but if we can evaluate the density of $f(x)$ or sample from the distribution of $x$, then we can obtain an arbitrarily good approximation of $\tilde{y}_{\mathrm{Bayes}}$ using Markov Chain Monte Carlo techniques.

\paragraph{Optimal frequentist denoising.}
Assumption A.2 spells out a minimax estimation problem, where the underlying parameter $f(x)$ is assumed to be within a set $S\subset \R^d$. In particular, we are interested in finding $\tilde{y}_{\mathrm{minimax}}$ that attains the minimax risk
$$
R(S) =  \min_{g:  \R^d \rightarrow \R^d }  \max_{f(x)\in S  } \E \left[\|g(\hat{y}) - f(x)\|^2\right],
$$
on $L_p$ balls $S = \cB(p,B) = \{ y \in \R^d \;|\;  \|y\|_p \leq B \}$ of radius $B$.

A complete characterization of this minimax risk (up to a constant) is given by \citet[Proposition~5]{birge2001gaussian}, who show that in the non-trivial region\footnote{When $\sqrt{\log d} \leq B/\sigma  \leq c_p d^{1/p}$ for a constant $c_p$ that depends only on $p$.} of the signal to noise ratio $B/\sigma$, the ball $S = \cB(p,B)$ satisfies
\begin{equation}\label{eq:minimax_pnorm_ball}
R(S)  = \Theta\left(B^p\sigma^{2-p} \left(1 + \log\left(\frac{d\sigma^p}{B^p}\right)\right)^{1-p/2}\right)
\end{equation}
for $0<p<2$ and when $p\geq 2$, \citet{donoho1990minimax} show that
$$R(S) = \Theta\left(\frac{B^2\sigma^2}{\sigma^2 + B^2/d}\right).$$

Deriving exact minimax estimators is challenging and most analyses assume certain asymptotic regimes (see the case for $p=2$ by \citet{bickel1981minimax}). Nonetheless, some techniques have been shown to match $R(\cB(p,B))$ up to a small constant factor in the finite sample regime \citep[see, e.g.,][]{donoho1990minimax,donoho1994minimax}.
This means that we can often improve the square error from $d \sigma^2$ to $R(\cB(p,B))$ when we have the additional information that $f(x)$ is in some $L_p$ ball. This could be especially helpful in the high-dimensional case for $p<2$. For instance if $p=1$ and $B=\sigma$, then we obtain a risk $\sigma B \sqrt{1+\log(d \sigma/B )}$, which improves exponentially in $d$ over the $d\sigma^2$ risk of $\hat{y}$. More practically, if $f(x)$ is a sparse histogram with $s$ non-zero elements, then taking $p\rightarrow 0$ will result in an error bound on the order of  $s \sigma^2 (1+\log(d))$, which is linear in the sparsity $s$ rather than the dimension $d$. 

\paragraph{Adaptive estimation.}
What if we do not know the prior parameter $w^2$, or a right choice of $B$ and $p$? Can we still come up with estimators that take advantage of these structures? It turns out that this is the problem of designing \emph{adaptive} estimators which sits at the heart of statistical research.
An \emph{adaptive} estimator in our case, is one that does not need to know $w^2$ or a pair of $B$ and $p$, yet behave nearly as well as Bayes estimator that knows $w^2$ or the minimax estimator that knows $B$ and $p$ for each parameter regime. 

We first give an example of an adaptive Bayes estimator that does not require us to specify a prior, yet can perform almost as well as the optimal Bayes estimator for all isotropic Gaussian prior simultaneously.
\begin{theorem}[James-Stein estimator and its adaptivity]\label{thm:js}
	When $d\geq3$, substituting $w^2$ in $\tilde{y}_{\mathrm{Bayes}}$ with its maximum likelihood estimate under 
	$$f(x)\sim \cN(0,w^2 I) \enspace, \;\; \hat{y} | f(x) \sim \cN(f(x), \sigma^2 I) $$
	produces the James-Stein estimator
	$$
	\tilde{y}_{\mathrm{JS}} = \left(  1-  \frac{(d-2)\sigma^2}{\|\hat{y}\|^2}\right)\hat{y}.
	$$
	Moreover, it has an MSE
	$$
	\E \left[\|\tilde{y}_{\mathrm{JS}} -  f(x)\|^2\right]   = d\sigma^2 \left(  1 - \frac{(d-2)^2}{d^2} \frac{\sigma^2}{w^2+\sigma^2} \right).
	$$
\end{theorem}
The James-Stein estimator has the property that it always improves the MLE $\hat{y}$ when $d\geq 3$ \citep{stein1956inadmissibility} and it always achieves a risk that is within a $d^2/(d-2)^2$ multiplicative factor of the Bayes risk of $\tilde{y}_{\mathrm{Bayes}}$
for any $w^2$. 
 
We now move on to describe a method that is adaptive to $B$ and $p$ in minimax estimation. Quite remarkably, \citet{donoho1995noising} shows that choosing $\lambda = \sigma\sqrt{2 \log d}$ in the soft-thresholding estimator
\begin{equation}\label{eq:soft-thresh}
\tilde{y}_{\mathrm{TH}}   =  \text{sign}(\hat{y})\max\{0,|\hat{y}|  - \lambda\}
\end{equation}
yields a nearly optimal estimator for every $L_p$ ball. 
\begin{theorem}[The adaptivity of soft-thresholding, Theorem 4.2 of \citep{donoho1995noising}]\label{thm:don}
	Let $S = \cB(p,B)$ for some $p, B > 0$. The soft-thresholding estimator with $\lambda = \sigma\sqrt{2 \log d}$ obeys that
	$$\sup_{f(x)\in S}\E\left[\|\tilde{y}_{\mathrm{TH}}  -  f(x)\|^2\right] \leq (2\log d + 1)(\sigma^2 + 2.22 R(S)) \enspace.$$
\end{theorem}
The result implies that the soft-thresholding estimator is nearly optimal for all balls up to a multiplicative factor of $4.44\log(d)$.

Thanks to the fact that we know the parameter $\sigma$ exactly, both $\tilde{y}_{\mathrm{JS}} $ and $\tilde{y}_{\mathrm{TH}}$ are now completely free of tuning parameters. Yet, they can achieve remarkable adaptivity that covers a large class of model assumptions and function classes. A relatively small price to pay for such adaptivity  is that we might lose a constant (or a $\log(d)$) factor. Whether such adaptivity is worth will vary on a case-by-case basis.

Take the problem of private releasing a histogram of $n$ items in $d$ bins. Theorem~\ref{thm:don} and Equation \eqref{eq:minimax_pnorm_ball} with $p\leq 1$ imply that the soft-thresholding estimator obeys
$$
\E \left[\|\tilde{y}_{\mathrm{TH}}  -  f(x)\|^2\right]  = \tilde{O} \left( \min\{s\sigma^2, n^{1/k}\sigma^{2-1/k}\}\right).
$$
where $s$ denotes the number of nonzero elements in $f(x)$ and $k$ is the largest power-law exponent greater than $1$ such that order statistics $f(x)^{(d-i+1)} \leq n i^{-k}$ for all $i=1,...,d$ and $\tilde{O}$ hides logarithmic factors in $d,  d\sigma/n$. 
The fact that $s \leq d$  implies that the soft-thresholding estimator improves over the naive private release for all $d,n,s$ and the $n^{1/k}$ factor suggests that we can take advantage of an unknown power law distribution even if the histogram is not really sparse. This makes our technique effective in the many data mining problems where power law distributions occur naturally \citep{faloutsos1999power}. 

\paragraph{Related work.}
Denoising as a post-processing step in the context of differentially privacy is not a new idea. Notably, \citet{barak2007privacy,hay2009accurate} show that a post-processing step enforcing consistency of contingency table releases and graph degree sequences leads to more accurate estimations. \citet{williams2010probabilistic} sets up the general statistical (Bayesian) inference problem of DP releases by integrating auxiliary information (a prior). \citet{karwa2016inference} exploits knowledge of the noise distribution use to achieve DP in the inference procedure of a network model and shows that it helps to preserve asymptotic efficiency. \citet{nikolov2013geometry} demonstrates that projecting linear regression solutions to a known $\ell_1$-ball improves the estimation error from $O(\mathrm{poly}(d))$ to $O(\mathrm{polylog}(d))$ when the underlying ground truth is sparse. \citet{bernstein2017differentially} uses Expectation--Maximization to denoise the parameters of a class of graphical models starting from noisy empirical moments obtained using the Laplace mechanism.

In all the above references there is some prior knowledge (constraint sets, sparsity or Bayesian prior) that is exploited to improve the utility of DP releases.
To the best of our knowledge, we are the first to consider ``adaptive estimation'' and demonstrate how classical techniques can be helpful even without such prior knowledge.
These estimators are not new; they have been known in the statistics literature for decades.
Our purpose is to compile facts that are relevant to the practice of DP and initiate a systematic study of how these ideas affect the utility of DP mechanisms, which we complement with the experimental evaluation presented in the next section.

\section{Numerical Experiments}\label{sec:exp}

This section provides an experimental evaluation of the improvements in utility provided by optimal calibration and adaptive denoising. First we numerically compare the variance of the analytic Gaussian mechanism and the classical mechanism for a variety of privacy parameters. Then we evaluate the contributions of denoising and analytic calibration against a series of baselines for the task of private mean estimation using synthetic data.
We also evaluate several denoising strategies on the task of releasing heat maps based on the New York City taxi dataset under differential privacy.
Further experiments are presented in Appendix~\ref{app:exp}, including an evaluation of denoising strategies for the task of
private histogram release.

\subsection{Analytic Gaussian Mechanism}\label{sec:expaGM}

\begin{figure*}[t]
\begin{center}
\begin{subfigure}[b]{0.244\textwidth}
\includegraphics[width=\textwidth]{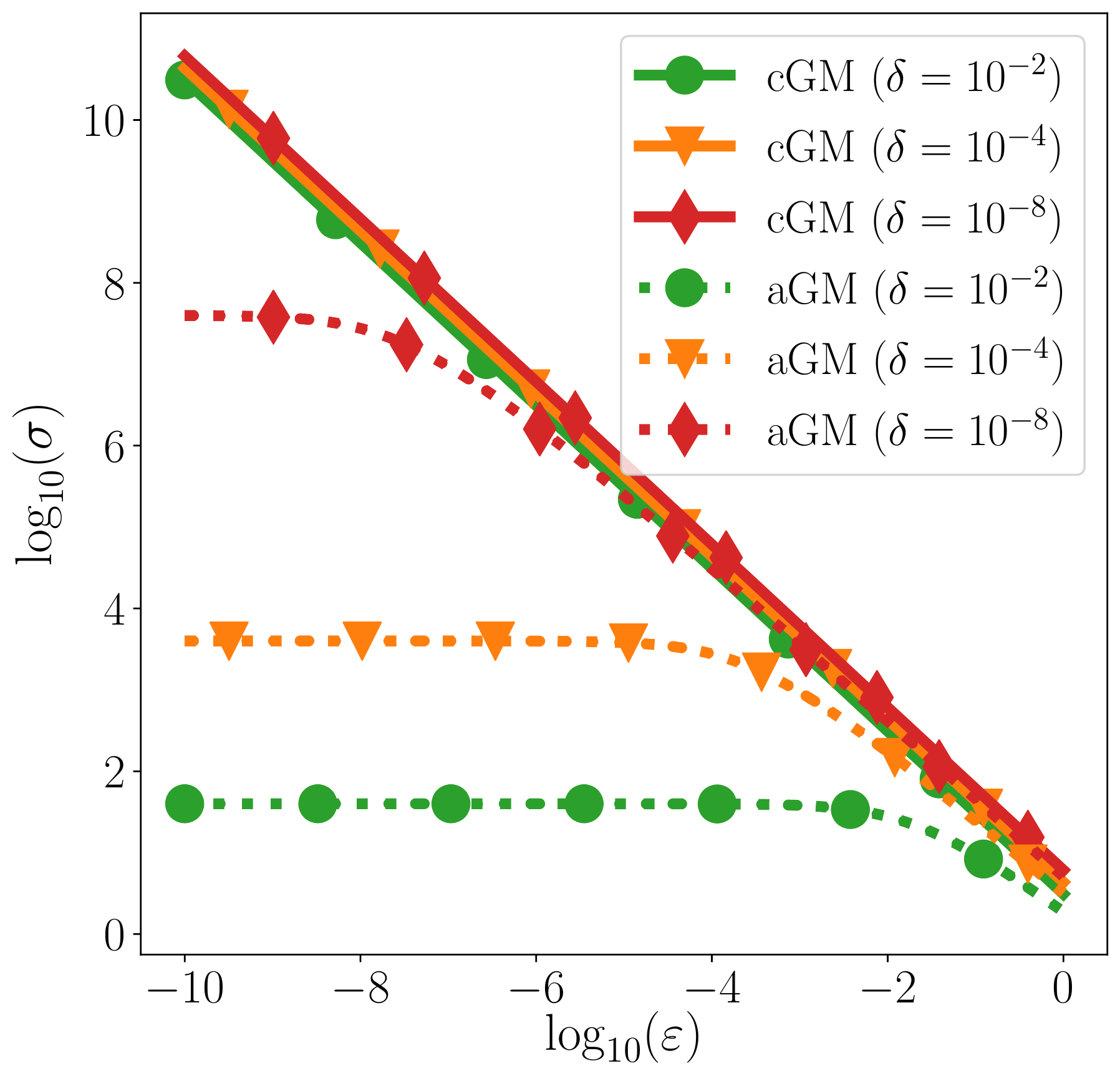}
\end{subfigure}
\begin{subfigure}[b]{0.244\textwidth}
\includegraphics[width=\textwidth]{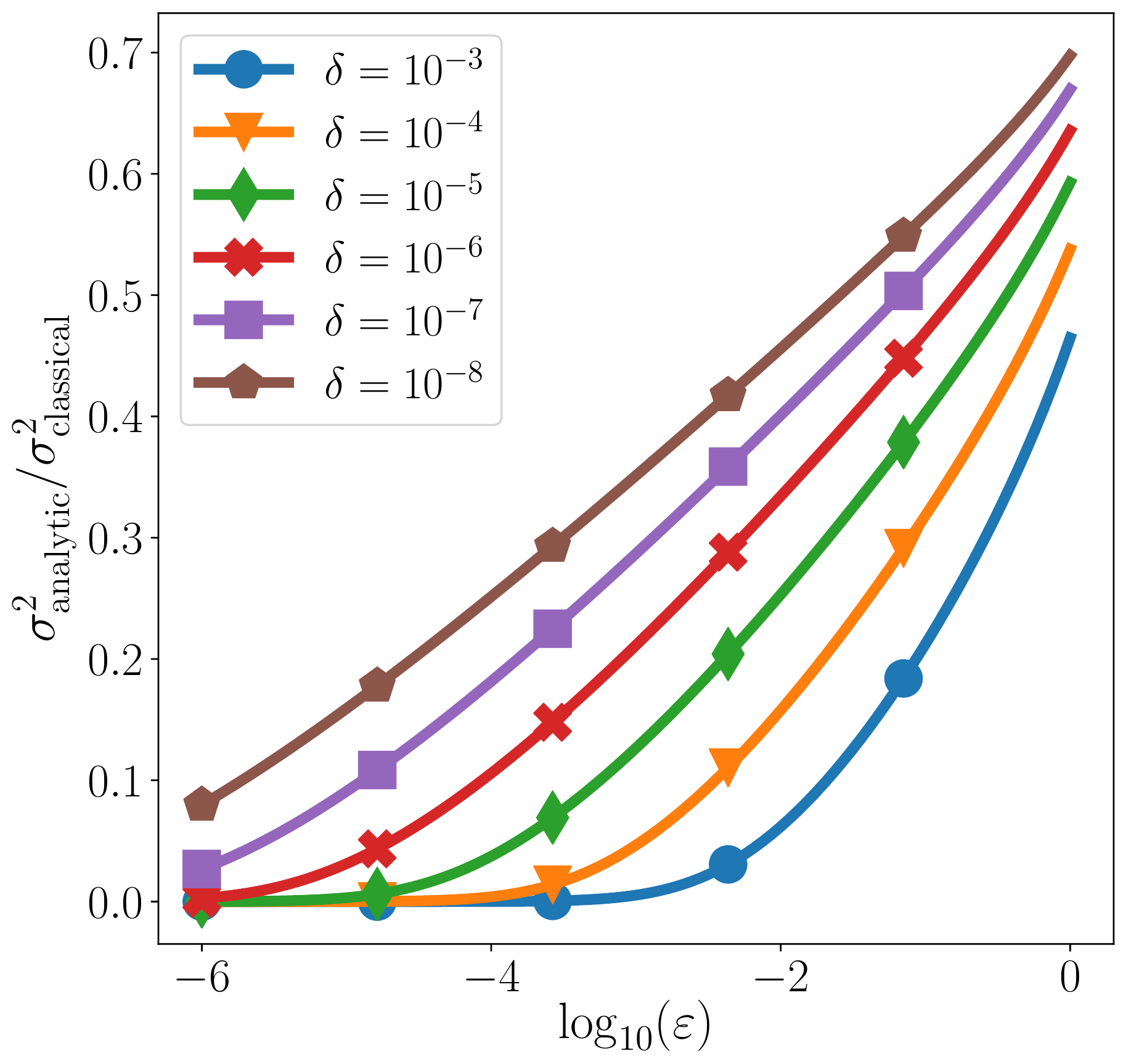}
\end{subfigure}
\begin{subfigure}[b]{0.244\textwidth}
\includegraphics[width=\textwidth]{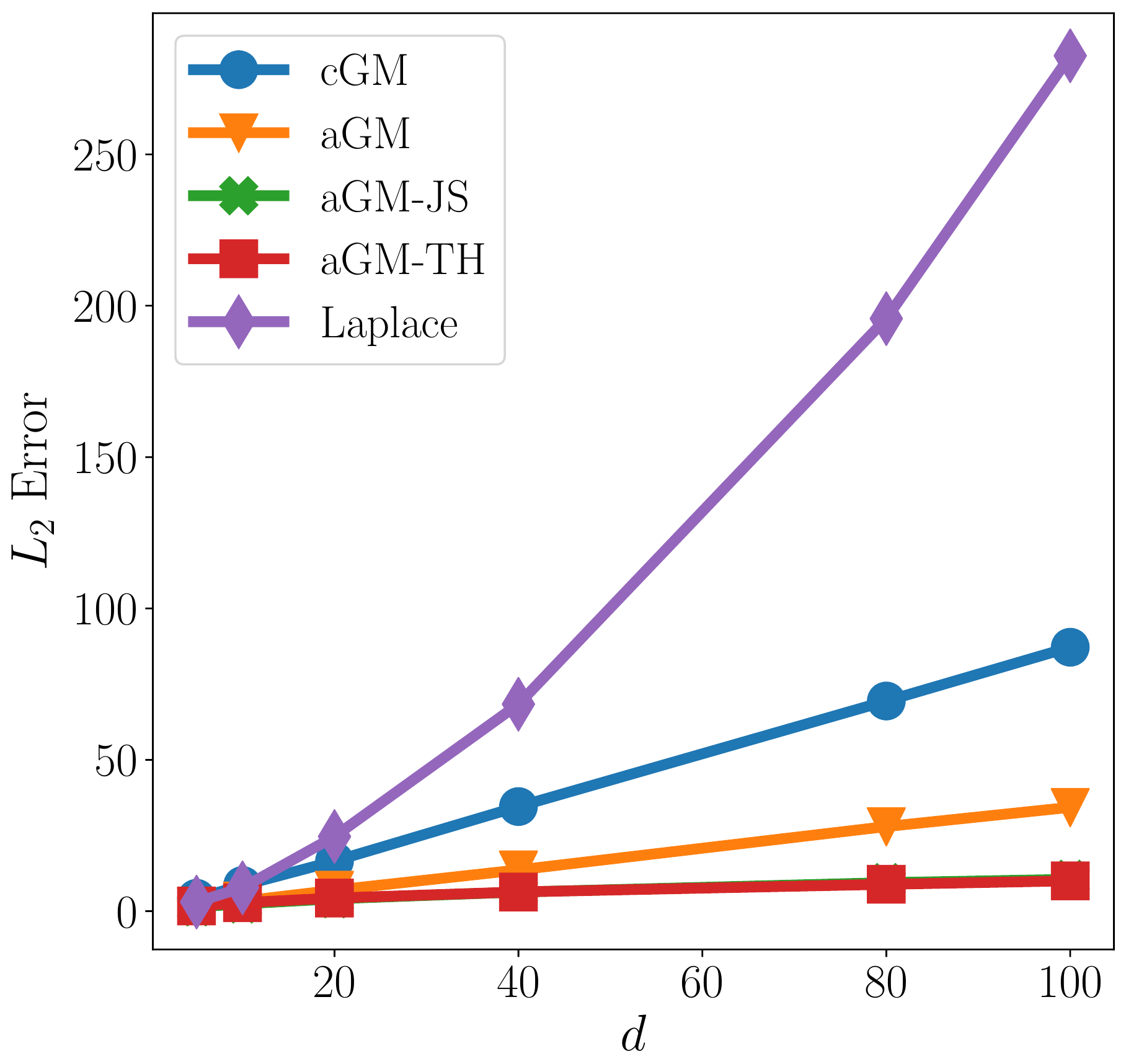}
\end{subfigure}
\begin{subfigure}[b]{0.244\textwidth}
\includegraphics[width=\textwidth]{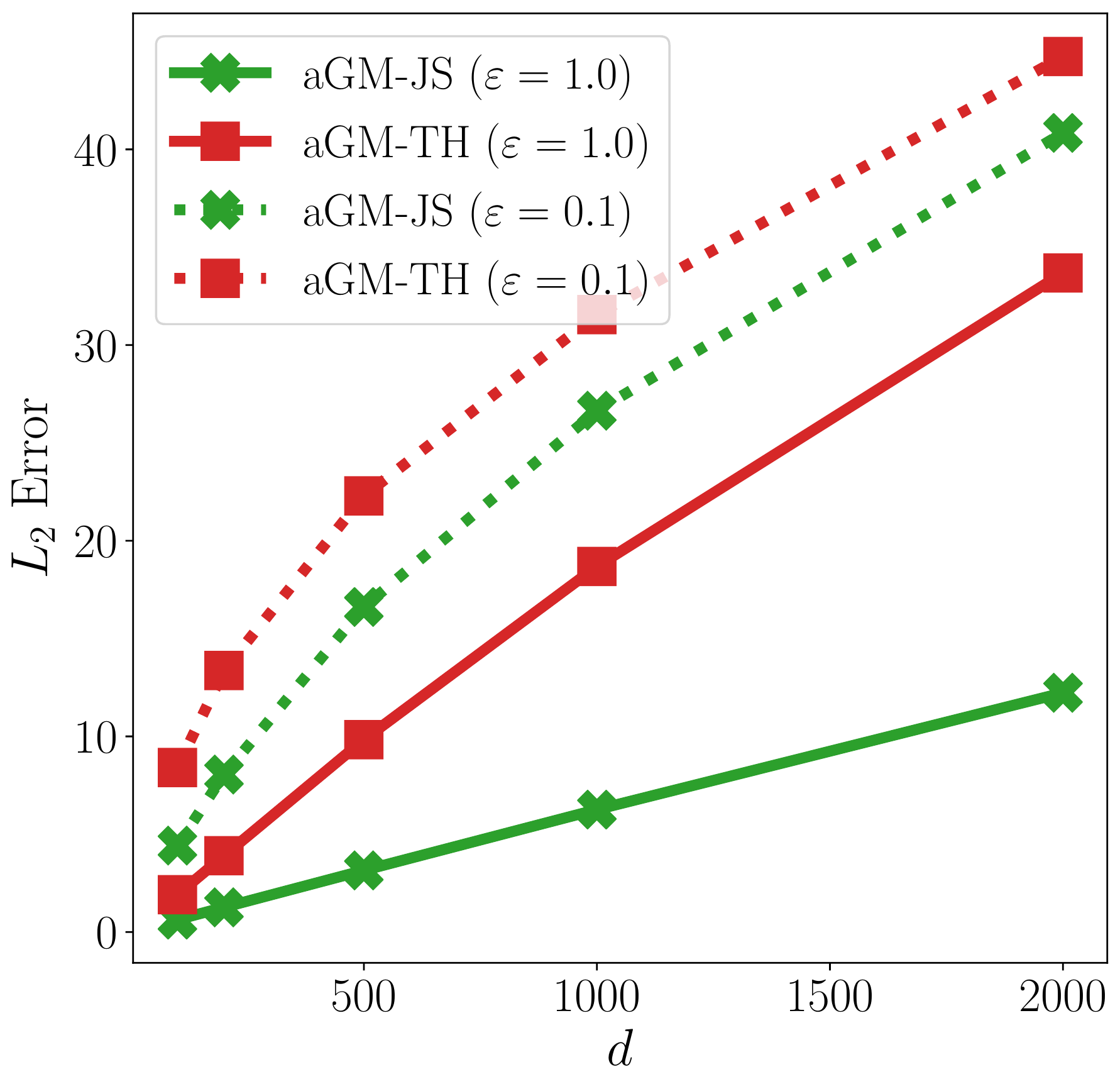}
\end{subfigure}
\caption{Two leftmost plots: Experiments comparing the classical Gaussian mechanism (cGM) and the analytic Gaussian mechanism (aGM), in terms of their absolute standard deviations as $\varepsilon \to 0$, and in terms of the gain in variance as a function of $\varepsilon$. Two rightmost plots: Mean estimation experiments showing $L_2$ error between the private mean estimate and the non-private empirical mean as a function of the dimension $d$.} \label{fig:exp-synth}
\end{center}
\end{figure*}

We implemented Algorithm~\ref{alg:newGM} in Python\footnote{See \url{https://github.com/BorjaBalle/analytic-gaussian-mechanism}.} and ran experiments to compare the variance of the perturbation obtained with the analytic Gaussian mechanism versus the variance required by the classical Gaussian mechanism. In all our experiments the values of $v^*$ and $u^*$ were solved up to an accuracy of $10^{-12}$ using binary search and the implementation of the $\mathsf{erf}$ function provided by SciPy \cite{scipy}.

The results are presented in the two leftmost panels in Figure~\ref{fig:exp-synth}. The plots show that as $\varepsilon \to 0$ the optimally calibrated perturbation outperforms the classical mechanism by several orders of magnitude.
Furthermore, we see that even for values of $\varepsilon$ close to $1$ our mechanism reduces the variance by a factor of $1.4$ or more, with higher improvements for larger values of $\delta$.

\subsection{Denoising for Mean Estimation}\label{sec:expmean}

Our next experiment evaluates the utility of denoising combined with the analytical Gaussian mechanism for the task of private mean estimation. The input to the mechanism is a dataset $x = (x_1, \ldots, x_n)$ containing $n$ vectors $x_i \in \R^d$ and the underlying deterministic functionality is $f(x) = (1/n) \sum_{i = 1}^n x_i$. This relatively simple task is a classic example from the family of \emph{linear queries} which are frequently considered in the differential privacy literature. We compare the accuracy of several mechanisms $M$ for releasing a private version of $f(x)$ in terms of the Euclidean distance $\norm{M(x) - f(x)}_2$. In particular, we test the analytical Gaussian mechanism with either James-Stein denoising cf.\ Theorem~\ref{thm:js} (aGM-JS) or optimal thresholding denoising cf.\ Theorem~\ref{thm:don} (aGM-TH), as well as several baselines including:  the classical Gaussian mechanism (cGM), the analytical Gaussian mechanism without denoising (aGM), and the Laplace mechanism (Lap) using the same $\varepsilon$ parameter as the Gaussian mechanisms.

To provide a thorough comparison we explore of the different parameters of the problem on the final utility. The key parameters of the problem are the dimension $d$ and the DP parameters $\varepsilon$ and $\delta$. The dimension affects the utility through the bounds provided in Theorem~\ref{thm:js} and Theorem~\ref{thm:don}. The DP parameters affect the utility through the variance $\sigma^2$ of the mechanism, which is also affected by the sample size $n$ via the global sensitivity. Thus, we can characterize the effect of $\sigma^2$ by keeping $n$ fixed and changing the DP parameters. In our experiments we consider a fixed sample size $n = 500$ and privacy parameter $\delta = 10^{-4}$ while trying several values for $\varepsilon$.

The other parameter that affects the utility is the ``size'' of $f(x)$, controlled either through the variance $w^2$ or the norm ball $S$. Since the denoising estimators we use are adaptive to these parameters and do not need to know them in advance, we sample the dataset $x$ repeatedly to obtain a diversity of values for $f(x)$. Each dataset $x$ is sampled as follows: first sample a center $x_0 \sim \cN(0,I)$ and then build $x = (x_1,\ldots,x_n)$ with $x_i = x_0 + \xi_i$, where each $\xi_i$ is i.i.d.\ with independent coordinates sampled uniformly from the interval $[-1/2,1/2]$. Thus, in each dataset the points $x_i$ all lie in an $L_{\infty}$-ball of radius $1$, leading to a global $L_2$ sensitivity $\Delta_2 = \sqrt{d}/n$ and a global $L_1$ sensitivity $\Delta_1 = d/n$. These are used to calibrate the Gaussian and Laplace perturbations, respectively.

The results are presented in two rightmost panels of Figure~\ref{fig:exp-synth}. Each point in every plot is the result of averaging the error over $100$ repetitions with different datasets. The first plot uses $\varepsilon = 0.01$ and shows how denoised methods improve the accuracy over all the other methods, sometimes by orders of magnitude. The second plot shows that for this problem the James-Stein estimator provides better accuracy in the high-dimensional setting.

\subsection{New York City Taxi Heat Maps}\label{sec:exptaxi}

\begin{figure*}
\captionsetup[subfigure]{justification=centering}
	\begin{subfigure}[t]{0.245\textwidth}
		\includegraphics[width=\textwidth]{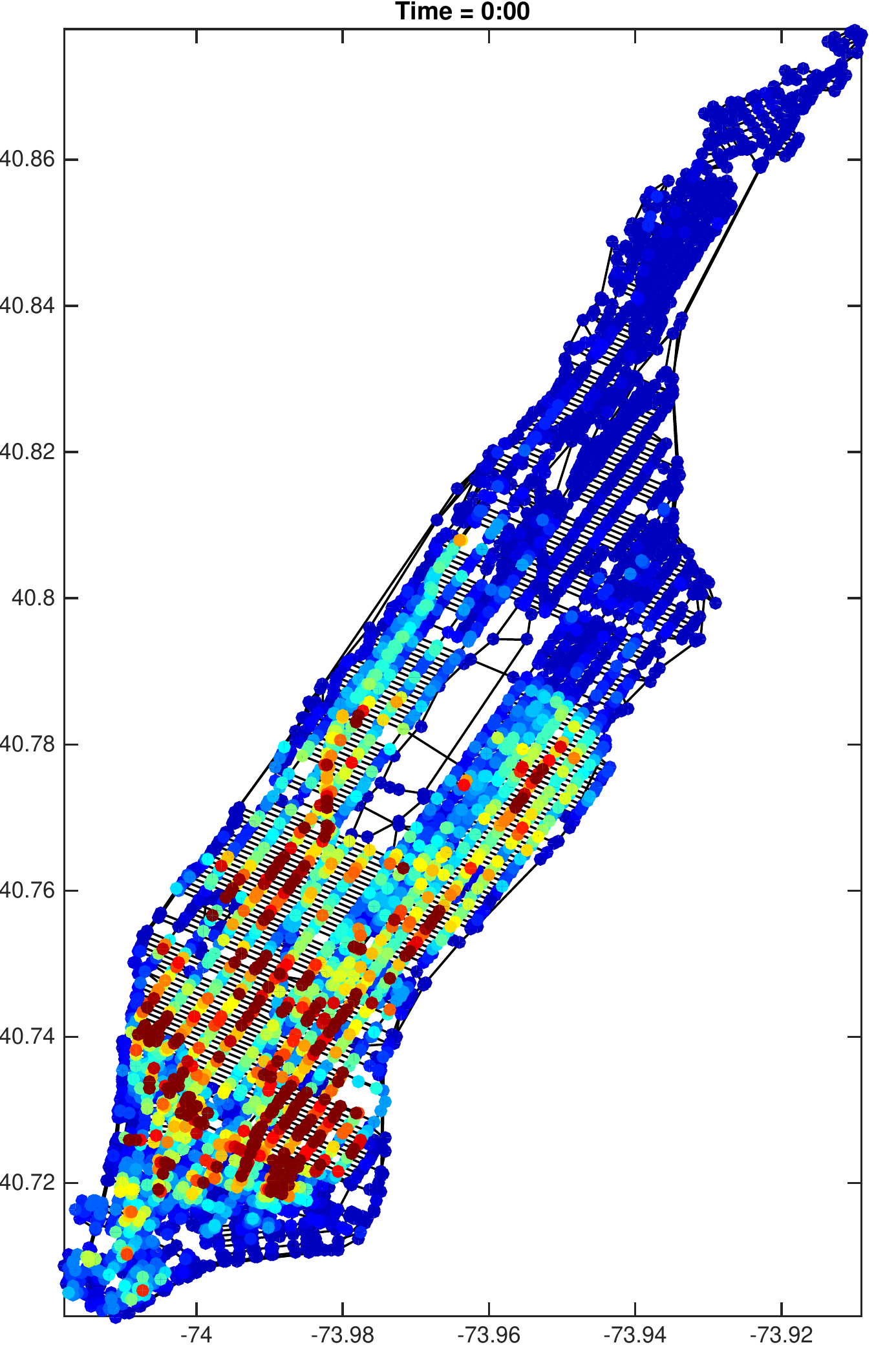}
		\caption*{Ground truth}
	\end{subfigure}
	\begin{subfigure}[t]{0.245\textwidth}
		\includegraphics[width=\textwidth]{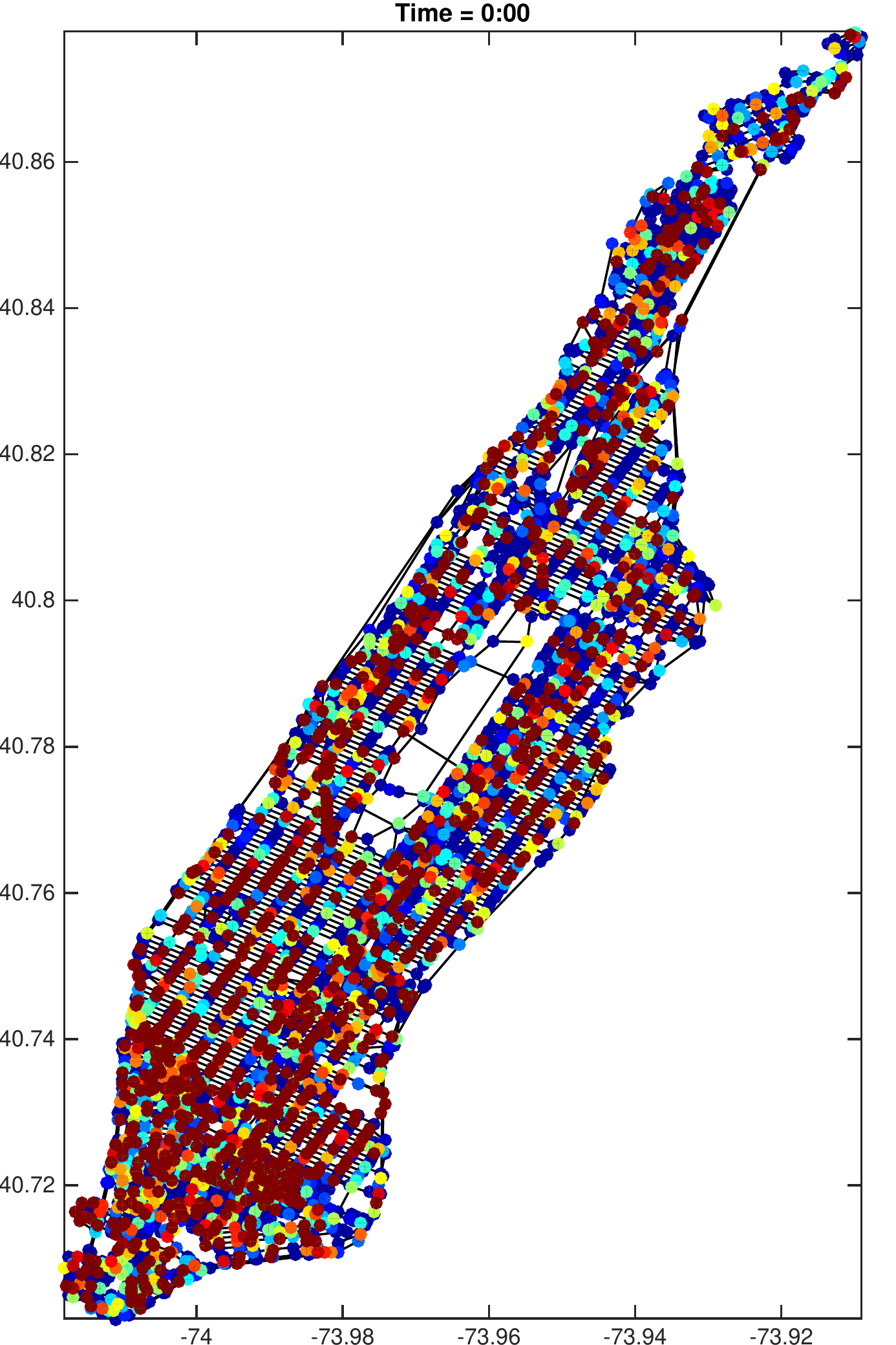}
		\caption*{Raw aGM\\ RMSE = 40.29}
	\end{subfigure}
	\begin{subfigure}[t]{0.245\textwidth}
		\includegraphics[width=\textwidth]{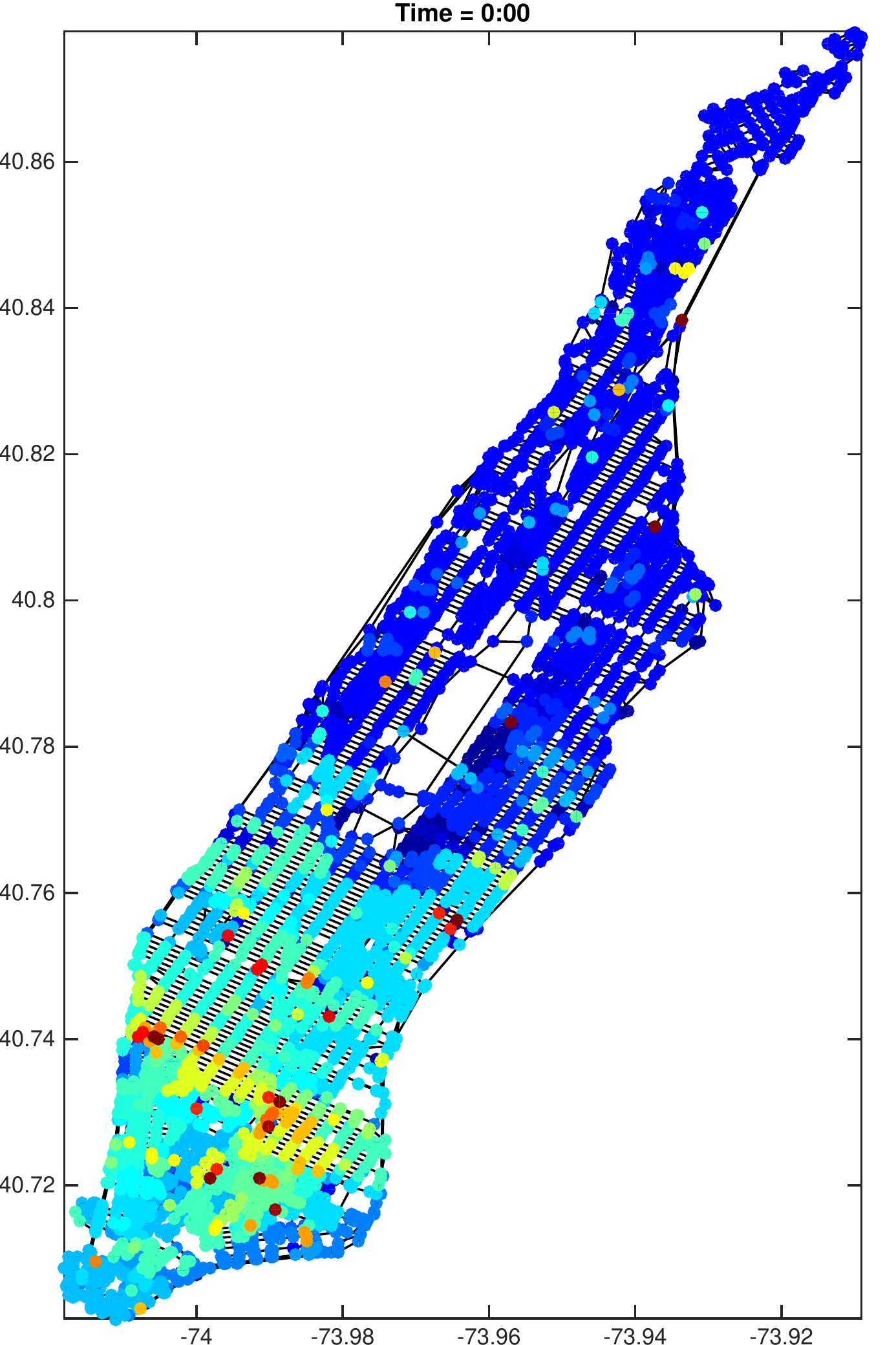}
		\caption*{Wavelets\\ RMSE = 12.28}
	\end{subfigure}
	\begin{subfigure}[t]{0.245\textwidth}
		\includegraphics[width=\textwidth]{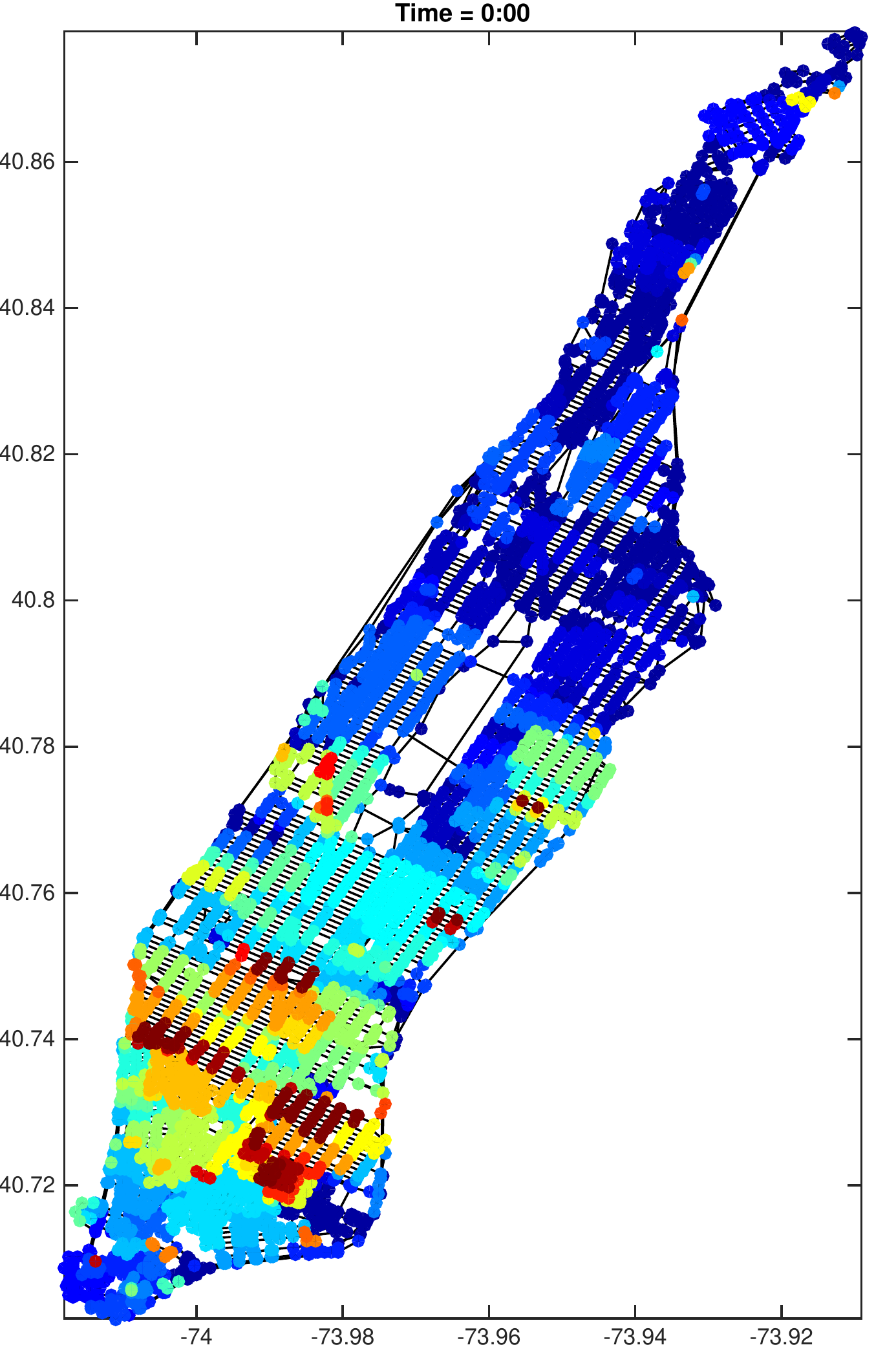}
		\caption*{Trend filtering\\ RMSE = 10.07}
	\end{subfigure}
	\caption{Illustration of the denoising in differentially private release of NYC taxi density during 00:00 - 01:00 am Sept 24, 2014. From left to right, the figures represent the true data, the output of the analytical Gaussian mechanism, the reconstructed signal from soft-thresholded wavelet coefficients with spanning-tree wavelet transform \citep{sharpnack2013detecting}, and the results of trend filtering on graphs \citep{wang2016}. We observe that adding appropriate post-processing significantly reduces the estimation error and also makes the underlying structures visible.
} \label{fig:illus_nyc_taxi}	
\end{figure*}

In this section, we apply our method to New York City taxi data. The dataset is a collection of time-stamped pick-ups and drop-offs of taxi drivers and we are interested in sharing a density map of such pick-ups and drop-offs in Manhattan at a specific time of a specific day under differential privacy.  

This is a problem of significant practical interest. Ever since the NYC Taxi \& Limousine Commission released this dataset, there has been multiple independent reports concerning the security and privacy risks this dataset poses for taxi drivers and their passengers \citep[see, e.g.,][]{taxi_rainbow,douriez2016anonymizing}. The techniques presented in this paper allow us to provably prevent individuals (on both the per-trip level and per-cab level) in the dataset from being identified, while remarkably, permitting the release of rich information about the data with fine-grained spatial and temporal resolution.

Specifically, we apply the analytical Gaussian mechanism to release the number of picks-ups and drop-offs at every traffic junction in Manhattan. There are a total of 3,784 such traffic junctions and they are connected by 7,070 sections of roads. We will treat them as nodes and edges on a graph. In the post-processing phase, we apply graph smoothing techniques to reveal the underlying signal despite the noise due to aGM. Specifically, we compare the JS-estimator and the soft-thresholding estimator we described in Section~\ref{sec:denoise}, as well as the same soft-thresholding estimator applied to the coefficients of a graph wavelet transform due to \citet{sharpnack2013detecting}. The basis transformation is important because the data might be sparser in the transformed domain. For reference, we also include the state-of-the-art graph smoothing techniques called graph trend filtering \citep{wang2016}, which has one additional tuning parameter but has been shown to perform significantly better than wavelet smoothing in practice. 

Our experiments provide cab-level differential privacy by assuming that every driver does a maximum of $5$ trips within an hour 
so that we have a global $L_2$-sensitivity of $\Delta = 5$.
This is a conservative but reasonable estimate and can be enforced by preprocessing the data.
Data within each hour is gathered and distributed to each traffic junction using a kernel density estimator; further details are documented in \citet{doraiswamy2014using}. 

We present some qualitative comparisons in Figure~\ref{fig:illus_nyc_taxi}, where we visualize the privately released heat map with and without post-processing.
Relatively speaking, trend filtering performs better than wavelet smoothing, but both approaches significantly improves the RMSE over the DP release without post-processing.
The results in Appendix~\ref{app:exp} provide quantitative results by comparing the mean square error of cGM, aGM as well as the aforementioned denoising techniques for data corresponding to different time intervals.

\section{Conclusion and Discussion}
In this paper, we embark on a journey of pushing the utility limit of Gaussian mechanism for $(\varepsilon,\delta)$-differential privacy. We propose a novel method to obtain the optimal calibration of Gaussian perturbations required to attain a given DP guarantee. We also review decades of research in statistical estimation theory and show that combining these techniques with differential privacy one obtains powerful \emph{adaptivity} that denoises differentially private outputs nearly optimally without additional hyperparameters. On synthetic data and on the New York City Taxi dataset we illustrate a significant gain in estimation error and fine-grained spatial-temporal resolution.

There are a number of theoretical problems of interest for future work. First, on the problem of differentially private estimation. Our post-processing approach effectively restricts our choice of algorithms to the composition of privacy release and post-processing. While we now know that we are optimal in both components, it is unclear whether we lose anything relative to the best differentially private algorithms.
Secondly, the analytical calibration proposed in this paper is optimal for achieving $(\varepsilon,\delta)$-DP with Gaussian noise. But when building complex mechanisms we are stuck in the dilemma of choosing between (a) using the aGM with the advanced composition \citep{kairouz2015composition}; or, (b) using R{\'e}nyi DP \citep{mironov2017renyi} or zCDP  \citep{bun2016concentrated} for tighter composition and calculate the $(\varepsilon,\delta)$ from moment bounds.
While (a) is tighter in the calculation the privacy parameters of each intermediate value, (b) is tighter in the composition but cannot take advantage of aGM.
It would be interesting if we could get the best of both worlds.

\section*{Acknowledgments}
We thank \citet{doraiswamy2014using} for sharing their preprocessed NYC taxi dataset, the anonymous reviewers for helpful comments that led to improvements of the paper and Stephen E. Fienberg for discussions that inspired the authors to think about optimal post-processing.

\bibliography{paper}

\begin{thebibliography}{33}
\providecommand{\natexlab}[1]{#1}
\providecommand{\url}[1]{\texttt{#1}}
\expandafter\ifx\csname urlstyle\endcsname\relax
  \providecommand{\doi}[1]{doi: #1}\else
  \providecommand{\doi}{doi: \begingroup \urlstyle{rm}\Url}\fi

\bibitem[Abadi et~al.(2016)Abadi, Chu, Goodfellow, McMahan, Mironov, Talwar,
  and Zhang]{abadi2016deep}
Mart{\'\i}n Abadi, Andy Chu, Ian Goodfellow, H~Brendan McMahan, Ilya Mironov,
  Kunal Talwar, and Li~Zhang.
\newblock Deep learning with differential privacy.
\newblock In \emph{Proceedings of the 2016 ACM SIGSAC Conference on Computer
  and Communications Security}, pages 308--318. ACM, 2016.

\bibitem[Barak et~al.(2007)Barak, Chaudhuri, Dwork, Kale, McSherry, and
  Talwar]{barak2007privacy}
Boaz Barak, Kamalika Chaudhuri, Cynthia Dwork, Satyen Kale, Frank McSherry, and
  Kunal Talwar.
\newblock Privacy, accuracy, and consistency too: a holistic solution to
  contingency table release.
\newblock In \emph{Proceedings of the twenty-sixth ACM SIGMOD-SIGACT-SIGART
  symposium on Principles of database systems}, pages 273--282. ACM, 2007.

\bibitem[Bassily et~al.(2014)Bassily, Smith, and Thakurta]{bassily2014private}
Raef Bassily, Adam Smith, and Abhradeep Thakurta.
\newblock Private empirical risk minimization: Efficient algorithms and tight
  error bounds.
\newblock In \emph{Foundations of Computer Science (FOCS), 2014 IEEE 55th
  Annual Symposium on}, pages 464--473. IEEE, 2014.

\bibitem[Bernstein et~al.(2017)Bernstein, McKenna, Sun, Sheldon, Hay, and
  Miklau]{bernstein2017differentially}
Garrett Bernstein, Ryan McKenna, Tao Sun, Daniel Sheldon, Michael Hay, and
  Gerome Miklau.
\newblock Differentially private learning of undirected graphical models using
  collective graphical models.
\newblock In \emph{International Conference on Machine Learning (ICML)}, 2017.

\bibitem[Bickel et~al.(1981)]{bickel1981minimax}
PJ~Bickel et~al.
\newblock Minimax estimation of the mean of a normal distribution when the
  parameter space is restricted.
\newblock \emph{The Annals of Statistics}, 9\penalty0 (6):\penalty0 1301--1309,
  1981.

\bibitem[Birg{\'e} and Massart(2001)]{birge2001gaussian}
Lucien Birg{\'e} and Pascal Massart.
\newblock Gaussian model selection.
\newblock \emph{Journal of the European Mathematical Society}, 3\penalty0
  (3):\penalty0 203--268, 2001.

\bibitem[Bun and Steinke(2016)]{bun2016concentrated}
Mark Bun and Thomas Steinke.
\newblock Concentrated differential privacy: Simplifications, extensions, and
  lower bounds.
\newblock In \emph{Theory of Cryptography Conference}, pages 635--658.
  Springer, 2016.

\bibitem[Donoho(1995)]{donoho1995noising}
David~L Donoho.
\newblock De-noising by soft-thresholding.
\newblock \emph{IEEE transactions on information theory}, 41\penalty0
  (3):\penalty0 613--627, 1995.

\bibitem[Donoho and Johnstone(1994)]{donoho1994minimax}
David~L Donoho and Iain~M Johnstone.
\newblock Minimax risk over p-balls for p-error.
\newblock \emph{Probability Theory and Related Fields}, 99\penalty0
  (2):\penalty0 277--303, 1994.

\bibitem[Donoho et~al.(1990)Donoho, Liu, and MacGibbon]{donoho1990minimax}
David~L Donoho, Richard~C Liu, and Brenda MacGibbon.
\newblock Minimax risk over hyperrectangles, and implications.
\newblock \emph{The Annals of Statistics}, pages 1416--1437, 1990.

\bibitem[Doraiswamy et~al.(2014)Doraiswamy, Ferreira, Damoulas, Freire, and
  Silva]{doraiswamy2014using}
Harish Doraiswamy, Nivan Ferreira, Theodoros Damoulas, Juliana Freire, and
  Cl{\'a}udio~T Silva.
\newblock Using topological analysis to support event-guided exploration in
  urban data.
\newblock \emph{IEEE transactions on visualization and computer graphics},
  20\penalty0 (12):\penalty0 2634--2643, 2014.

\bibitem[Douriez et~al.(2016)Douriez, Doraiswamy, Freire, and
  Silva]{douriez2016anonymizing}
Marie Douriez, Harish Doraiswamy, Juliana Freire, and Cl{\'a}udio~T Silva.
\newblock Anonymizing nyc taxi data: Does it matter?
\newblock In \emph{Data Science and Advanced Analytics (DSAA), 2016 IEEE
  International Conference on}, pages 140--148. IEEE, 2016.

\bibitem[Dwork and Roth(2014)]{dwork2014algorithmic}
Cynthia Dwork and Aaron Roth.
\newblock The algorithmic foundations of differential privacy.
\newblock \emph{Foundations and Trends in Theoretical Computer Science},
  9\penalty0 (3-4):\penalty0 211--407, 2014.

\bibitem[Dwork and Rothblum(2016)]{dwork2016concentrated}
Cynthia Dwork and Guy~N Rothblum.
\newblock Concentrated differential privacy.
\newblock \emph{arXiv preprint arXiv:1603.01887}, 2016.

\bibitem[Dwork et~al.(2006)Dwork, McSherry, Nissim, and
  Smith]{dwork2006calibrating}
Cynthia Dwork, Frank McSherry, Kobbi Nissim, and Adam Smith.
\newblock Calibrating noise to sensitivity in private data analysis.
\newblock In \emph{Theory of cryptography}, pages 265--284. Springer, 2006.

\bibitem[Dwork et~al.(2009)Dwork, Naor, Reingold, Rothblum, and
  Vadhan]{dwork2009complexity}
Cynthia Dwork, Moni Naor, Omer Reingold, Guy~N Rothblum, and Salil Vadhan.
\newblock On the complexity of differentially private data release: efficient
  algorithms and hardness results.
\newblock In \emph{Proceedings of the forty-first annual ACM symposium on
  Theory of computing}, pages 381--390. ACM, 2009.

\bibitem[Faloutsos et~al.(1999)Faloutsos, Faloutsos, and
  Faloutsos]{faloutsos1999power}
Michalis Faloutsos, Petros Faloutsos, and Christos Faloutsos.
\newblock On power-law relationships of the internet topology.
\newblock In \emph{ACM SIGCOMM computer communication review}, volume~29, pages
  251--262. ACM, 1999.

\bibitem[Gordon(1941)]{gordon1941values}
Robert~D Gordon.
\newblock Values of mills' ratio of area to bounding ordinate and of the normal
  probability integral for large values of the argument.
\newblock \emph{The Annals of Mathematical Statistics}, 12\penalty0
  (3):\penalty0 364--366, 1941.

\bibitem[Hay et~al.(2009)Hay, Li, Miklau, and Jensen]{hay2009accurate}
Michael Hay, Chao Li, Gerome Miklau, and David Jensen.
\newblock Accurate estimation of the degree distribution of private networks.
\newblock In \emph{Data Mining, 2009. ICDM'09. Ninth IEEE International
  Conference on}, pages 169--178. IEEE, 2009.

\bibitem[Jones et~al.(2001)Jones, Oliphant, Peterson, et~al.]{scipy}
Eric Jones, Travis Oliphant, Pearu Peterson, et~al.
\newblock {SciPy}: Open source scientific tools for {Python}, 2001.
\newblock URL \url{http://www.scipy.org/}.

\bibitem[Kairouz et~al.(2015)Kairouz, Oh, and
  Viswanath]{kairouz2015composition}
Peter Kairouz, Sewoong Oh, and Pramod Viswanath.
\newblock The composition theorem for differential privacy.
\newblock In \emph{International Conference on Machine Learning (ICML)}, 2015.

\bibitem[Karwa et~al.(2016)Karwa, Slavkovi{\'c}, et~al.]{karwa2016inference}
Vishesh Karwa, Aleksandra Slavkovi{\'c}, et~al.
\newblock Inference using noisy degrees: Differentially private {$\beta$}-model
  and synthetic graphs.
\newblock \emph{The Annals of Statistics}, 44\penalty0 (1):\penalty0 87--112,
  2016.

\bibitem[Lehmann and Casella(2006)]{lehmann2006theory}
Erich~L Lehmann and George Casella.
\newblock \emph{Theory of point estimation}.
\newblock Springer Science \& Business Media, 2006.

\bibitem[Lyu et~al.(2017)Lyu, Su, and Li]{lyu2017understanding}
Min Lyu, Dong Su, and Ninghui Li.
\newblock Understanding the sparse vector technique for differential privacy.
\newblock \emph{Proceedings of the VLDB Endowment}, 10\penalty0 (6):\penalty0
  637--648, 2017.

\bibitem[Mironov(2017)]{mironov2017renyi}
Ilya Mironov.
\newblock Renyi differential privacy.
\newblock In \emph{Computer Security Foundations Symposium (CSF), 2017 IEEE
  30th}, pages 263--275. IEEE, 2017.

\bibitem[Nikolov et~al.(2013)Nikolov, Talwar, and Zhang]{nikolov2013geometry}
Aleksandar Nikolov, Kunal Talwar, and Li~Zhang.
\newblock The geometry of differential privacy: the sparse and approximate
  cases.
\newblock In \emph{Proceedings of the forty-fifth annual ACM symposium on
  Theory of computing}, pages 351--360. ACM, 2013.

\bibitem[Pandurangan()]{taxi_rainbow}
Vijay Pandurangan.
\newblock On taxi and rainbows.
\newblock \url{https://tech.vijayp.ca/of-taxis-and-rainbows-f6bc289679a1}.
\newblock Accessed: 2014-06-21.

\bibitem[Rogers et~al.(2016)Rogers, Roth, Ullman, and
  Vadhan]{rogers2016privacy}
Ryan~M Rogers, Aaron Roth, Jonathan Ullman, and Salil Vadhan.
\newblock Privacy odometers and filters: Pay-as-you-go composition.
\newblock In \emph{Advances in Neural Information Processing Systems}, pages
  1921--1929, 2016.

\bibitem[Sharpnack et~al.(2013)Sharpnack, Singh, and
  Krishnamurthy]{sharpnack2013detecting}
James Sharpnack, Aarti Singh, and Akshay Krishnamurthy.
\newblock Detecting activations over graphs using spanning tree wavelet bases.
\newblock In \emph{Artificial Intelligence and Statistics}, pages 536--544,
  2013.

\bibitem[Stein(1956)]{stein1956inadmissibility}
Charles Stein.
\newblock Inadmissibility of the usual estimator for the mean of a multivariate
  normal distribution.
\newblock In \emph{Proceedings of the Third Berkeley symposium on mathematical
  statistics and probability}, volume~1, pages 197--206, 1956.

\bibitem[Telgarsky(2013)]{DBLP:journals/corr/abs-1301-4917}
Matus Telgarsky.
\newblock Dirichlet draws are sparse with high probability.
\newblock \emph{CoRR}, abs/1301.4917, 2013.

\bibitem[Wang et~al.(2016)Wang, Sharpnack, Smola, and Tibshirani]{wang2016}
Yu-Xiang Wang, James Sharpnack, Alexander~J. Smola, and Ryan~J. Tibshirani.
\newblock Trend filtering on graphs.
\newblock \emph{Journal of Machine Learning Research}, 17\penalty0
  (105):\penalty0 1--41, 2016.

\bibitem[Williams and McSherry(2010)]{williams2010probabilistic}
Oliver Williams and Frank McSherry.
\newblock Probabilistic inference and differential privacy.
\newblock In \emph{Advances in Neural Information Processing Systems}, pages
  2451--2459, 2010.

\end{thebibliography}
\bibliographystyle{plainnat}

\clearpage
\newpage
\appendix

\section{Proofs}\label{app:proofs}
In this appendix we present supporting proofs for all the results mentioned in the main text.

\subsection{Proofs from Section~\ref{sec:limitations}}

\begin{proof}[Proof of Theorem~\ref{thm:TV}]
An simple way to see that $(0,\delta)$-DP is achievable with Gaussian noise is to recall that $(0,\delta)$-DP is equivalent to a bound of $\delta$ on the total variation (TV) distance between the output distributions of $M(x)$ and $M(x')$ for any neighbouring pair $x \simeq x'$. If $M(x)$ is an output perturbation mechanism for $f(x)$ with noise $Z \sim \cN(0,\sigma^2 I)$, then using Pinsker's inequality we have
\begin{align*}
\mathsf{TV}(M(x),M(x')) &\leq \sqrt{\frac{\mathsf{KL}(M(x) | M(x'))}{2}} \\
&=\sqrt{\frac{\mathsf{KL}(\cN(f(x),\sigma^2 I) | \cN(f(x'),\sigma^2 I))}{2}} \\
&= \frac{\norm{f(x) - f(x')}}{2 \sigma} \leq \frac{\Delta}{2 \sigma} \enspace.
\end{align*}
Thus, we see that a Gaussian perturbation with standard deviation $\sigma = \Delta / 2 \delta$ is enough to achieve $(0,\delta)$-DP.
\end{proof}

\begin{proof}[Proof of Theorem~\ref{thm:lowerb}]
Note that the proof of Theorem~\ref{thm:aGM} shows that a Gaussian perturbation with $\sigma = \Delta / \sqrt{2 \varepsilon}$ yields a $(\varepsilon,\delta_0(\varepsilon))$-DP mechanism, where $\delta_0(\varepsilon) = \Phi(0) - e^{\varepsilon} \Phi(-\sqrt{2 \varepsilon})$. Thus, it is not possible to attain $(\varepsilon,\delta)$-DP with $\delta < \delta_0(\varepsilon)$ without increasing the variance of the perturbation.

The result follows by showing that the upper bound for $\delta$ proposed in Theorem~\ref{thm:lowerb} is a lower bound for $\delta_0(\varepsilon)$. Since $\Phi(0) = 1/2$, all we need to show is $e^{\varepsilon} \Phi(-\sqrt{2 \varepsilon}) < \frac{e^{-3 \varepsilon}}{\sqrt{4 \pi \varepsilon}}$.

Let $\Phi^{c}(t) = \Pr[\cN(0,1) \geq t] = 1 - \Phi(t)$ be the complementary of the standard Gaussian CDF. The Mill's ratio for the Gaussian distribution is the quantity $r(t) = \sqrt{2\pi} e^{t^2/2} \Phi^{c}(t)$. Bounding the Mill's ratio is a standard approach to approximate the tail of the Gaussian distribution. A well-known bound for the Mill's ratio is Gordon's inequality $r(t) < 1/t$ \cite{gordon1941values}. By using the symmetry $\Phi(-t) = \Phi^{c}(t)$ we obtain :
\begin{align*}
e^{\varepsilon} \Phi(-\sqrt{2 \varepsilon})
= e^{\varepsilon} \Phi^{c}(\sqrt{2 \varepsilon})
= \frac{e^{-3 \varepsilon}}{\sqrt{2 \pi}} r(\sqrt{2 \varepsilon})
< \frac{e^{-3 \varepsilon}}{\sqrt{4 \pi \varepsilon}} \enspace.
\end{align*}
\end{proof}

\begin{proof}[Proof of Lemma~\ref{lem:privlossGM}]
Recall that the density of the Gaussian output perturbation mechanism $M(x) = f(x) + Z$ with $Z \sim \cN(0,\sigma^2 I)$ is given by $p_{M(x)}(y) = \exp(-\norm{y - f(x)}^2/2 \sigma^2) / \sqrt{2 \pi \sigma^2}$. Plugging this expression into the definition of the privacy loss function and performing a quick computation we get
\begin{align*}
\ell_{M,x,x'}(y)
&=
\frac{\norm{y - f(x')}^2 - \norm{y - f(x)}^2}{2 \sigma^2} \\
&=
\frac{\norm{f(x) - f(x')}^2}{2 \sigma^2} + \frac{\langle y - f(x), f(x) - f(x') \rangle}{\sigma^2}
\enspace.
\end{align*}
To compute the privacy loss random variable $L_{M,x,x'}$ we need to plug $Y = f(x) + Z$ with $Z \sim \cN(0, \sigma^2 I)$ in the above inner product. By observing that $\langle Z, f(x) - f(x') \rangle \sim \cN(0, \sigma^2 \norm{f(x) - f(x')}^2)$ we obtain the distribution of the privacy loss random variable is given by
\begin{align*}
L_{M,x,x'} \sim \cN\left(\frac{\norm{f(x)-f(x')}^2}{2\sigma^2},\frac{\norm{f(x)-f(x')}^2}{\sigma^2}\right) \enspace.
\end{align*}
Therefore, the privacy loss of the Gaussian mechanism has the form $\cN(\eta, 2 \eta)$ for $\eta = D^2 / 2 \sigma^2$.
\end{proof}

\subsection{Proofs from Section~\ref{sec:aGM}}

\begin{proof}[Proof of Theorem~\ref{thm:necesuff}]
Given a pair of neighbouring datasets $x \simeq x'$ let $p = p_{M(x)}$ and $p' = p_{M(x')}$ be the densities of the output random variables $Y = M(x)$ and $Y' = M(x')$. Note that given an event $E \subseteq \Y$ one can rewrite \eqref{eqn:DP} as follows:
\begin{align}\label{eqn:DPevent}
\int_E (p(y) - e^\varepsilon p'(y)) \leq \delta \enspace.
\end{align}
Defining the event $E_* = \{ y : p(y) \geq e^\varepsilon p'(y) \}$ and its complementary $\bar{E}_* = \Y \setminus E_*$, we can partition $E$ into the sets $E_+ = E \cap E_*$ and $E_- = E \cap \bar{E}_*$. Therefore, by the definition of $E_*$ we have
\begin{align*}
\int_E (p(y) - e^\varepsilon p'(y)) &= \int_{E_+} (p(y) - e^\varepsilon p'(y)) \\ &\;\; + \int_{E_-} (p(y) - e^\varepsilon p'(y)) \\
&\leq \int_{E_+} (p(y) - e^\varepsilon p'(y)) \\
&\leq \int_{E_*} (p(y) - e^\varepsilon p'(y)) \enspace.
\end{align*}
Because \eqref{eqn:DPevent} has to hold for any event $E$ and the upper bound above holds for any event, we conclude that $M$ is $(\varepsilon,\delta)$-DP if and only if
\begin{align}\label{eqn:intEstar}
\int_{E_*} (p(y) - e^\varepsilon p'(y)) \leq \delta
\end{align}
holds for any $x \simeq x'$. To complete the proof we need to show that \eqref{eqn:intEstar} is equivalent to \eqref{eqn:necesuff}. Expanding the definition of $L_{M,x,x'}$ we get:
\begin{align*}
\Pr[L_{M,x,x'} \geq \varepsilon]
&=
\Pr[\log(p(Y) / p'(Y)) \geq \varepsilon] \\
&=
\Pr[p(Y) \geq e^{\varepsilon} p'(Y)] \\
&=
\int_{\Y} \one[p(y) \geq e^{\varepsilon} p'(y)] p(y) \\
&=
\int_{E_*} p(y) \enspace.
\end{align*}
A similar argument with $L_{M,x',x}$ also shows:
\begin{align*}
\Pr[L_{M,x',x} \leq -\varepsilon] = \int_{E_*} p'(y) \enspace.
\end{align*}
Putting the last two equations together we obtain see that the left hand side of \eqref{eqn:necesuff} equals the left hand side of \eqref{eqn:intEstar}.
\end{proof}

\begin{proof}[Proof of Lemma~\ref{lem:LtoCDF}]
Note that Lemma~\ref{lem:privlossGM} shows that the privacy loss random variables $L_{M,x,x'}$ and $L_{M,x'x}$ both follow the same distribution $\cN(\eta,2 \eta)$ with $\eta = D^2/2 \sigma^2$. This allows us to write the left hand side of \eqref{eqn:L1} in terms of the Gaussian CDF $\Phi$ as follows:
\begin{align*}
\Pr[L_{M,x,x'} \geq \varepsilon]
&=
\Pr[\cN(\eta, 2 \eta) \geq \varepsilon] \\
&=
\Pr\left[\cN(0,1) \geq \frac{-\eta + \varepsilon}{\sqrt{2 \eta}}\right] \\
&=
\Pr\left[\cN(0,1) \leq \frac{\eta - \varepsilon}{\sqrt{2 \eta}}\right] \\
&=
\Phi\left(\sqrt{\frac{\eta}{2}} - \frac{\varepsilon}{\sqrt{2 \eta}}\right) \\
&=
\Phi\left(\frac{D}{2 \sigma} - \frac{ \varepsilon \sigma}{D}\right) \enspace,
\end{align*}
where we used $\cN(\eta,2\eta) = \eta + \cN(0,1)/\sqrt{2 \eta}$ and the symmetry $\Pr[\cN(0,1) \geq t] = \Pr[\cN(0,1) \leq - t]$ of the distribution of the Gaussian distribution around its mean. A similar argument applied to the left hand side of \eqref{eqn:L2} yields:
\begin{align*}
\Pr[L_{M,x',x} \leq - \varepsilon]
&=
\Phi\left(- \frac{D}{2 \sigma} - \frac{ \varepsilon \sigma}{D}\right) \enspace.
\end{align*}
\end{proof}

\begin{proof}[Proof of Lemma~\ref{lem:monotone}]
We prove the result by using Leibniz's rule for differentiation under the integral sign to show that the function of interest has non-negative derivatives. First note that from the derivation of \eqref{eqn:L1} we have
\begin{align*}
\Pr[\cN(\eta,2\eta) \geq \varepsilon]
&=
\Phi(a(\eta))
=
\frac{1}{\sqrt{2 \pi}} \int_{-\infty}^{a(\eta)} e^{-y^2/2} dy \enspace,
\end{align*}
where $a(\eta) = \sqrt{\eta/2} - \varepsilon/\sqrt{2 \eta}$. Now we can use Leibniz's rule to write
\begin{align*}
\frac{d}{d \eta} \int_{-\infty}^{a(\eta)} e^{-y^2/2} dy
&=
e^{-a(\eta)^2/2} a'(\eta) \\
&=
e^{-a(\eta)^2/2} \left(\frac{1}{\sqrt{8 \eta}} + \frac{\varepsilon}{\sqrt{8 \eta^3}}\right) \enspace.
\end{align*}
Similarly, for the second term in the function we have $\Pr[\cN(\eta,2\eta) \leq - \varepsilon] = \Phi(b(\eta))$ where $b(\eta) = -\sqrt{\eta/2} - \varepsilon/\sqrt{2 \eta}$. Using Leibniz's rule again we get
\begin{align*}
\frac{d}{d \eta} \int_{-\infty}^{b(\eta)} e^{-y^2/2} dy
&=
e^{-b(\eta)^2/2} \left(-\frac{1}{\sqrt{8 \eta}} + \frac{\varepsilon}{\sqrt{8 \eta^3}}\right) \enspace.
\end{align*}
Therefore, we see that the derivative of $h$ satisfies:
\begin{align*}
h'(\eta)
&=
\frac{1}{4 \sqrt{\pi \eta}} \left(e^{-a(\eta)^2/2} + e^{\varepsilon} e^{-b(\eta)^2/2}\right) \\
&\;\; +
\frac{\varepsilon}{4 \sqrt{\pi \eta^3}} \left(e^{-a(\eta)^2/2} - e^{\varepsilon} e^{-b(\eta)^2/2}\right) \\
&= \frac{1}{4 \sqrt{\pi \eta}} \left(e^{-a(\eta)^2/2} + e^{\varepsilon} e^{-b(\eta)^2/2}\right) \geq 0
\enspace,
\end{align*}
where we used that $a(\eta)^2 + 2 \varepsilon = b(\eta)^2$.
\end{proof}

\begin{proof}[Proof of Theorem~\ref{thm:aGM}]
Recall that the derivations in Section~\ref{sec:aGM} establish that in order to calibrate a Gaussian perturbation to achieve $(\varepsilon,\delta)$-DP all that is required is find the smallest $\sigma$ such that
\begin{align}
\Phi\left(\frac{\Delta}{2 \sigma} - \frac{\varepsilon \sigma}{\Delta}\right) - e^{\varepsilon} \Phi\left(-\frac{\Delta}{2 \sigma} - \frac{\varepsilon \sigma}{\Delta}\right) \leq \delta \enspace.
\label{eqn:Nbis}
\end{align}
To establish the correctness of the analytic Gaussian mechanism we begin by observing that the argument in the first term of \eqref{eqn:Nbis} changes sign at $\sigma = \Delta / \sqrt{2 \varepsilon}$, while the argument for the second terms is always negative. Thus, we substitute $\sigma = \alpha \Delta / \sqrt{2 \varepsilon}$ in the expression above and obtain:
\begin{align*}\label{eqn:B}
B_{\varepsilon}(\alpha) = \Phi\left(\sqrt{\frac{\varepsilon}{2}} \left(\frac{1}{\alpha} - \alpha\right)\right)
- e^{\varepsilon} \Phi\left(-\sqrt{\frac{\varepsilon}{2}} \left(\frac{1}{\alpha} + \alpha\right)\right)
\enspace.
\end{align*}
To solve the optimization $\inf\{ \alpha > 0 : B_{\varepsilon}(\alpha) \leq \delta\}$ using numerical evaluations of $\Phi$ it is convenient to consider the cases $\alpha \geq 1$ and $\alpha < 1$ separately.
In the case $\alpha \geq 1$ we define $u = (\alpha - 1/\alpha)^2/2$ and substitute the corresponding $\alpha$ in $B_{\varepsilon}$ to obtain
\begin{align*}
B_{\varepsilon}^-(u) = \Phi(-\sqrt{\varepsilon u})
- e^{\varepsilon} \Phi(-\sqrt{\varepsilon (u + 2)}) \enspace.
\end{align*}
Similarly, by taking $v = (1/\alpha - \alpha)^2/2$ in the case $\alpha < 1$ we obtain 
\begin{align*}
B_{\varepsilon}^+(v) = \Phi(\sqrt{\varepsilon v})
- e^{\varepsilon} \Phi(-\sqrt{\varepsilon (v + 2)}) \enspace.
\end{align*}
Note that, as expected, these definitions satisfy $\lim_{v \to \infty} B_{\varepsilon}^+(v) = 1$ and $\lim_{u \to \infty} B_{\varepsilon}^-(u) = 0$, since the limits correspond to $\lim_{\alpha \to 0} B(\alpha) = 1$ and $\lim_{\alpha \to \infty} B(\alpha) = 0$, respectively.  Furthermore, we have
\begin{align*}
B_{\varepsilon}^+(0) = B_{\varepsilon}^-(0) = \Phi(0) - e^{\varepsilon} \Phi(- \sqrt{2 \varepsilon}) = \delta_0(\varepsilon) \enspace,
\end{align*}
which corresponds to the privacy guarantee $(\varepsilon,\delta_0(\varepsilon))$-DP obtained by taking $\sigma = \Delta / \sqrt{2 \varepsilon}$; i.e.\ $\alpha = 1$.

These observations motivate the mechanism described in Algorithm~\ref{alg:newGM}. In particular, for $\delta \geq \delta_0(\varepsilon)$ we can achieve $(\varepsilon,\delta)$-DP with $\alpha < 1$, and the smallest $\alpha < 1$ such that $B_{\varepsilon}(\alpha) \leq \delta$ corresponds to the largest $v \geq 0$ such that $B_{\varepsilon}^+(v) \leq \delta$. Similarly, for $\delta < \delta_0(\varepsilon)$ we require $\alpha \geq 1$, and the smallest $\alpha \geq 1$ such that $B_{\varepsilon}(\alpha) \leq \delta$ corresponds to the smallest $u \geq 0$ such that $B_{\varepsilon}^-(u) \leq \delta$.
\end{proof}

\subsection{Proofs from Section~\ref{sec:denoise}}
The proofs in this section are well-known and not part of the contribution of the current paper. We include these proofs because they are short and revealing and we hope to be self-contained as much as possible.
\begin{proof}[Proof of Theorem~\ref{thm:bayes}]
	 Let $P$ be the distribution of $f(x)$ induced by  $x\sim \pi$.
	Let $\theta\in \R^d$, define its posterior error 
	$$r(\theta | \hat{y})   =  \int \|\theta -  f(x)\|^2  dP(f(x) | \hat{y}).$$
	Take the gradient with respect to $\theta$ on both sides and apply Fubini's theorem 
	\begin{align*}
	 &\frac{\partial }{\partial \theta}r(\theta | \hat{y})  = \frac{\partial }{\partial \theta}\int\|\theta -  f(x)\|^2  dP(f(x) | \hat{y})\\
	   =&  	2\theta - 2  \int  f(x) dP(f(x) | \hat{y})  = 2(\theta - \E[ f(x) | \hat{y}]).
	\end{align*}
	Assign the gradient to $0$ we get that the minimizer is $\E[ f(x) | \hat{y}]$.
	
	Now, assume $\tilde{y}_{\mathrm{Bayes}} $ is suboptimal, there exists $\tilde{y}^*\neq \tilde{y}_{\mathrm{Bayes}} $ such that
	\begin{align*}
	&\E\left[\|f(x) - \tilde{y}_{\mathrm{Bayes}}\|^2 \right]  > \E\left[\|f(x) - \tilde{y}^*\|^2\right] \\
	& = \E[  r(\tilde{y}^* | \hat{y})] \geq \E[ r(\tilde{y}_{\mathrm{Bayes}} )]  =  \E\left[\|f(x) - \tilde{y}_{\mathrm{Bayes}}\|^2\right].
	\end{align*}
	which is a contradiction.
\end{proof}

\begin{proof}[Proof of Theorem~\ref{thm:js}]
	Note that $\frac{\|\hat{y}\|^2}{w^2+\sigma^2}$ follows a $\chi^2$ distribution with degree of freedom $d$.  The likelihood function
	$$
	p(\|\hat{y}\|^2 |  w^2)   \propto   (\frac{\|\hat{y}\|^2}{w^2+\sigma^2})^{d/2-1}e^{-\frac{\|\hat{y}\|^2}{2(w^2+\sigma^2)}}.
	$$
	The gradient w.r.t. $w^2$ of the log-likelihood, we get
	$$
	-\frac{d/2-1}{w^2+\sigma^2} + \frac{\|y\|^2}{2(w^2+\sigma^2)^2}.
	$$
	Assigning it to $0$, we get the maximum likelihood estimate $w^2  =  \frac{\|y\|^2}{k-2} - \sigma^2.$
	Substituting it into $\tilde{y}_{\mathrm{Bayes}}   =  (w^2 / (w^2+\sigma^2)) \hat{y}$
	produces $\tilde{y}_{\mathrm{JS}} $ as stated and the calculation of its MSE is straightforward.
\end{proof}

\section{Additional Experiments}\label{app:exp}
\begin{figure*}[t]
\begin{center}
\begin{subfigure}[b]{0.245\textwidth}
\includegraphics[width=\textwidth]{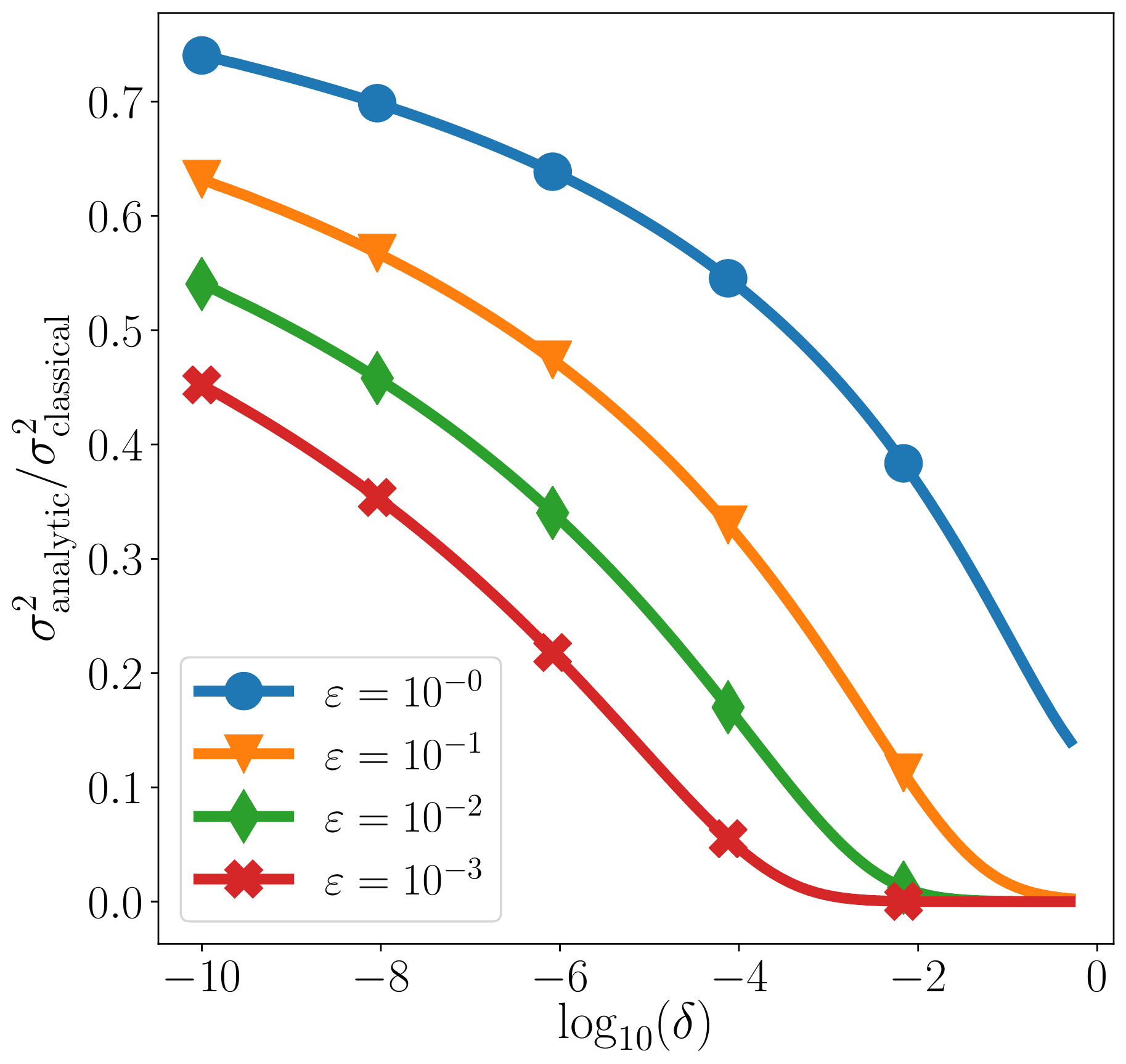}
\end{subfigure}
\begin{subfigure}[b]{0.245\textwidth}
\includegraphics[width=\textwidth]{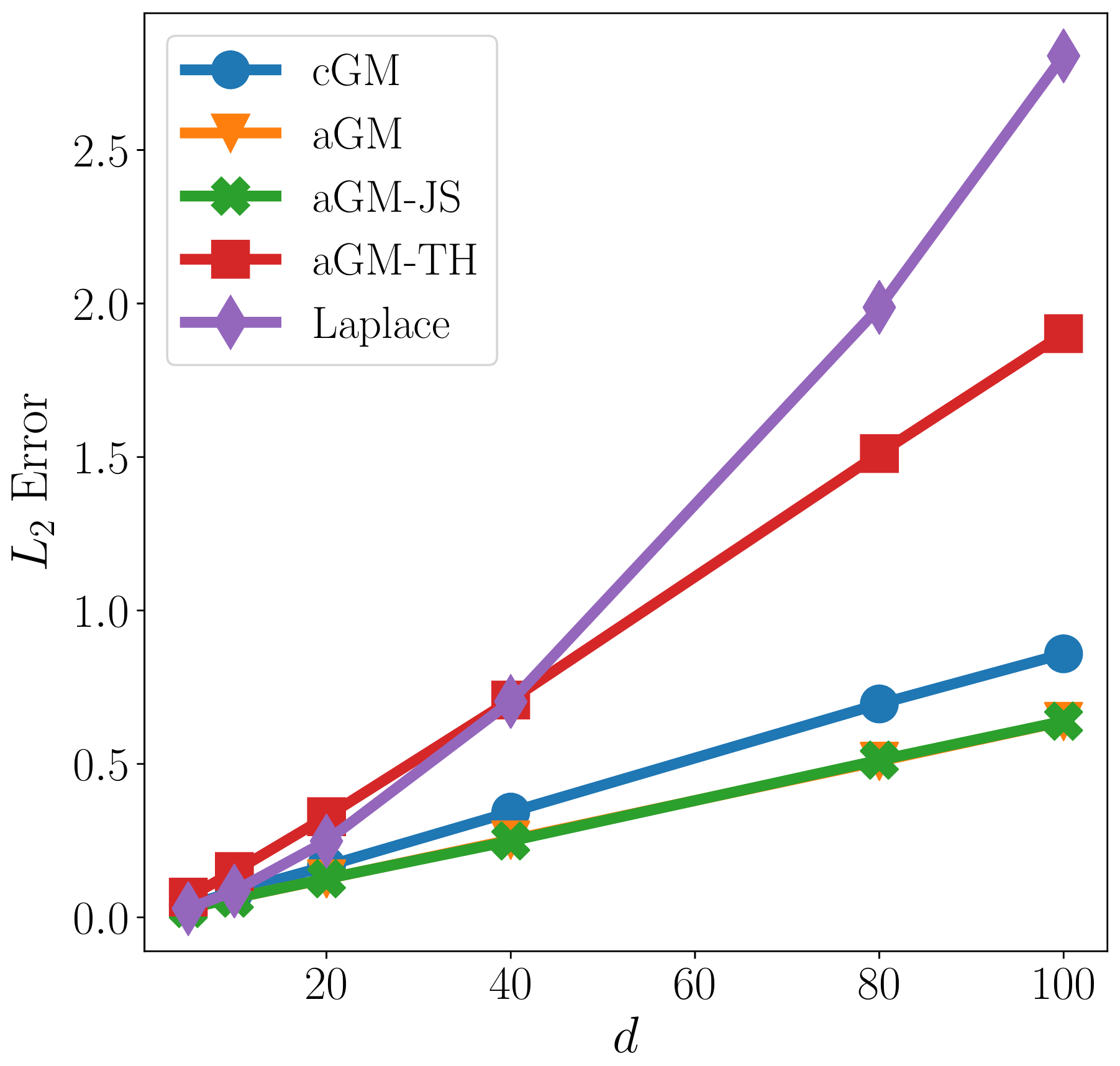}
\end{subfigure}
\begin{subfigure}[b]{0.245\textwidth}
\includegraphics[width=\textwidth]{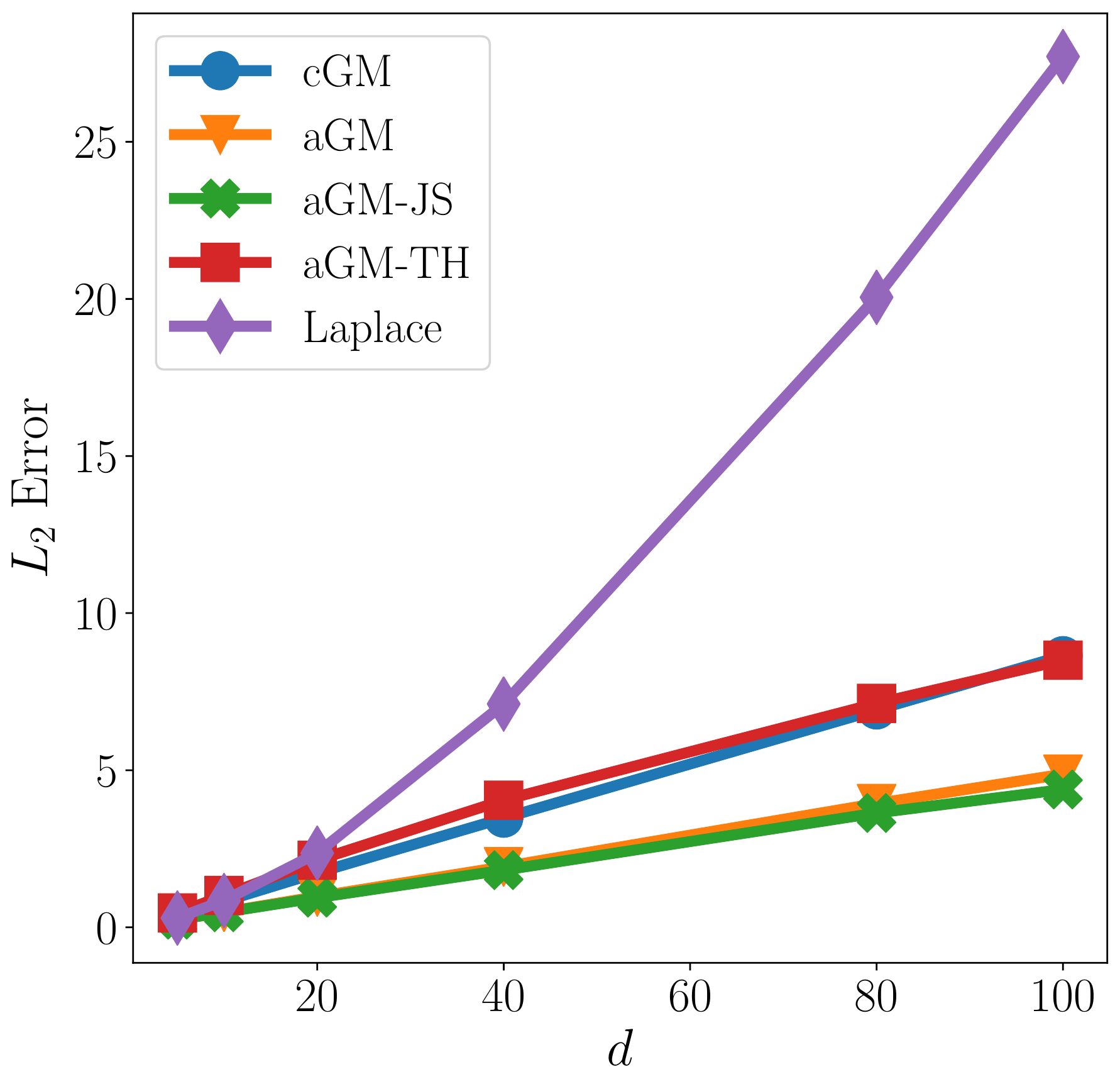}
\end{subfigure}
\caption{Left plot: Comparing the classical Gaussian mechanism (cGM) and the analytic Gaussian mechanism (aGM) in terms of gain in variance as a function of $\delta$. Two rightmost plots: Mean estimation experiments showing $L_2$ error between the private mean estimate and the non-private empirical mean as a function of the dimension $d$ with $\varepsilon = 1$ and $\varepsilon = 0.1$. Dataset size is fixed to $n = 500$ and privacy parameter is set to $\delta = 10^{-4}$.} \label{fig:exp-synth-extra}
\end{center}
\end{figure*}
\begin{figure*}[t]
\begin{center}
\begin{subfigure}[b]{0.244\textwidth}
\includegraphics[width=\textwidth]{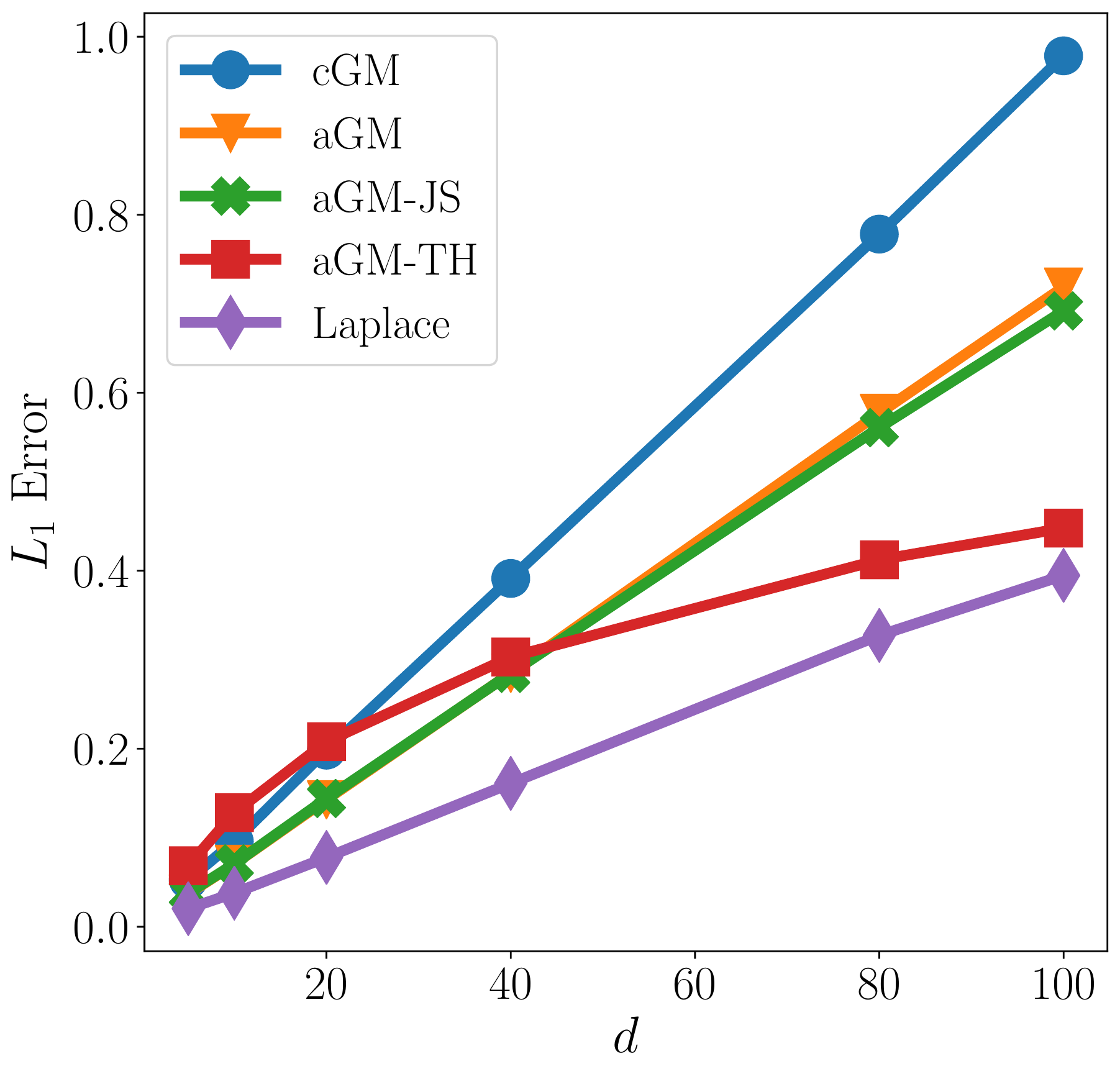}
\end{subfigure}
\begin{subfigure}[b]{0.244\textwidth}
\includegraphics[width=\textwidth]{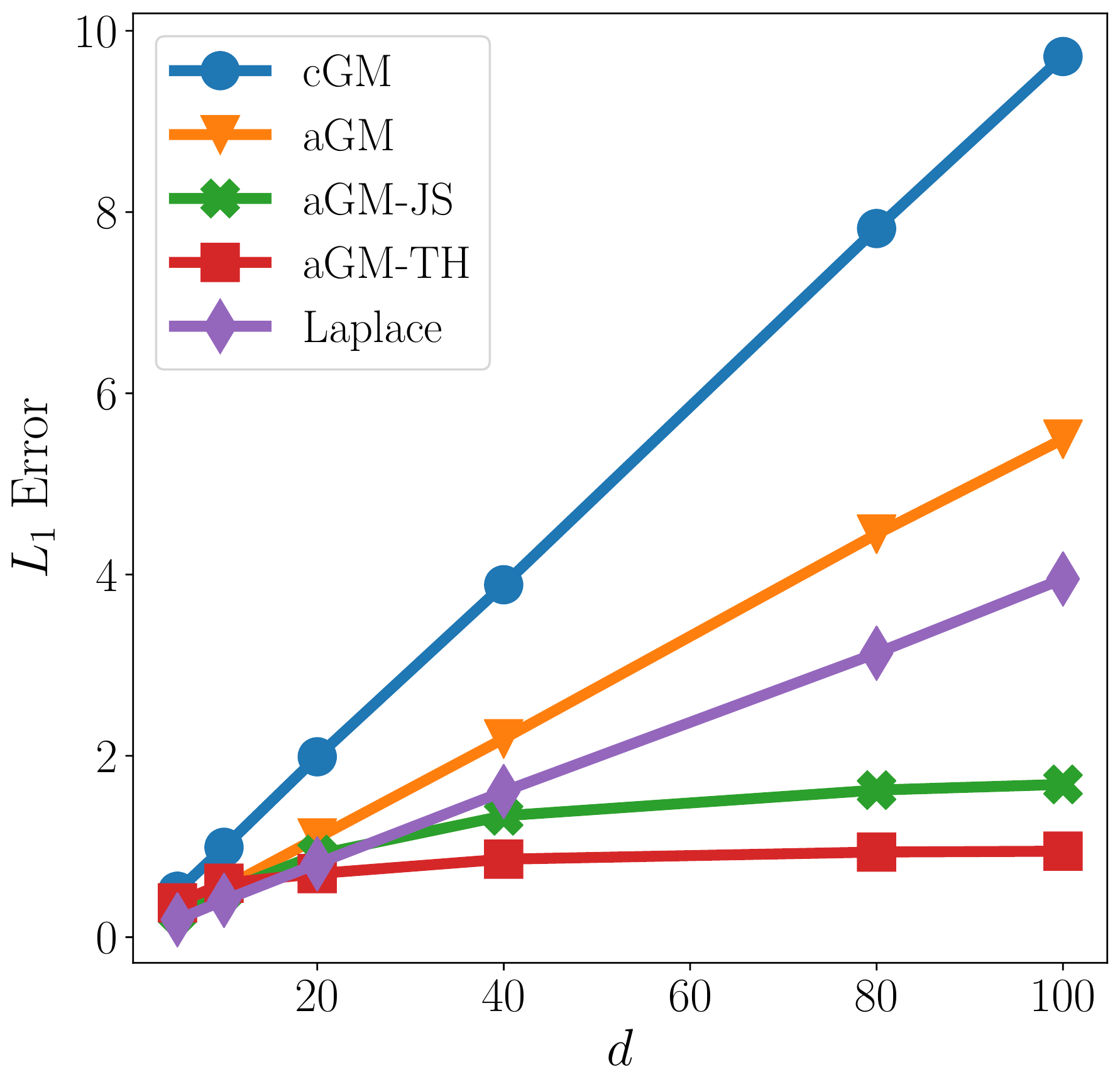}
\end{subfigure}
\begin{subfigure}[b]{0.244\textwidth}
\includegraphics[width=\textwidth]{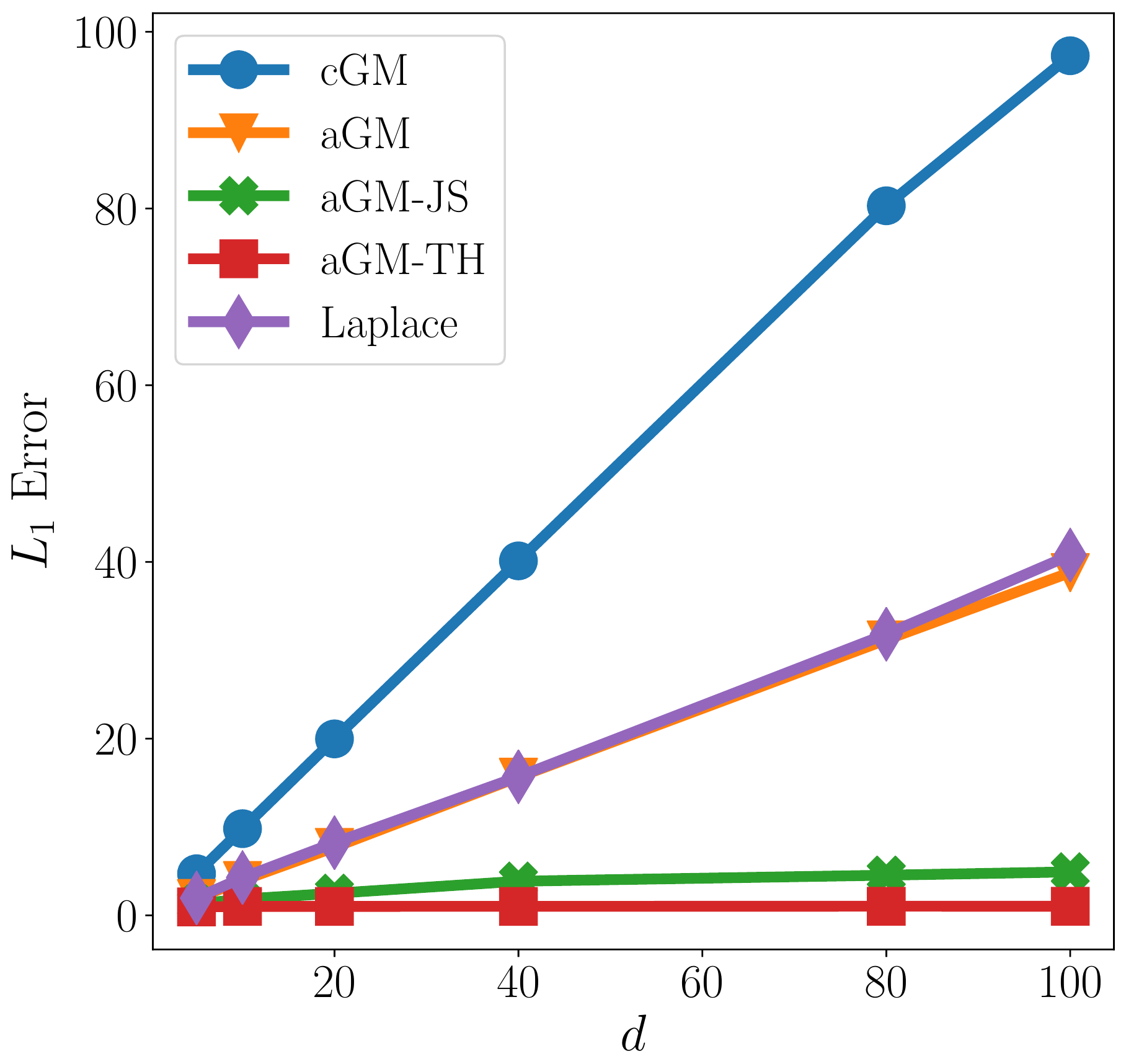}
\end{subfigure}
\begin{subfigure}[b]{0.244\textwidth}
\includegraphics[width=\textwidth]{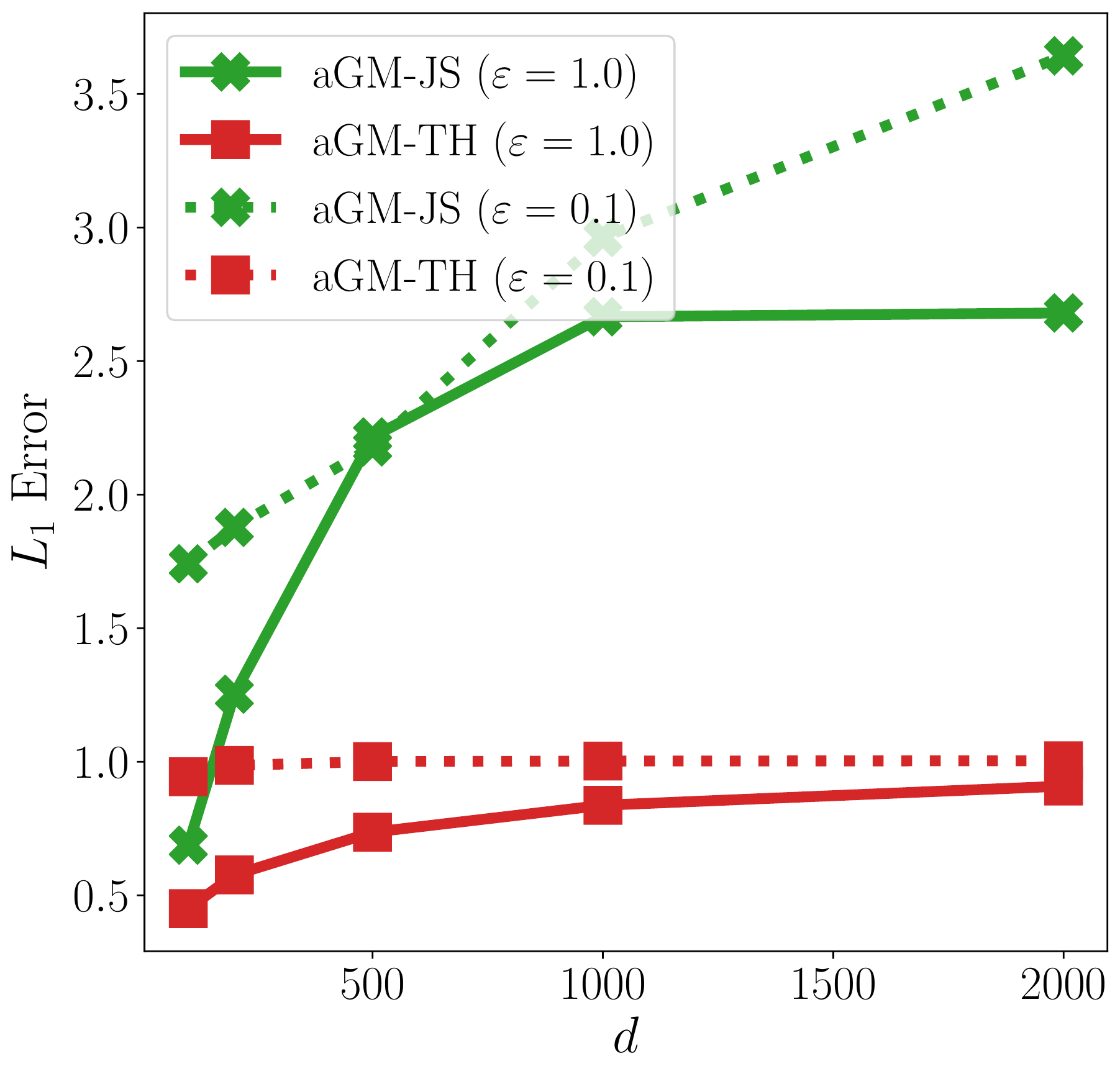}
\end{subfigure}
\caption{Histogram release experiments showing $L_1$ error between the private histogram and the non-private empirical histogram as a function of the dimension $d$. Dataset size is fixed to $n = 500$ and privacy parameter is set to $\delta = 10^{-4}$. The first three panels correspond to $\varepsilon = 1, 0.1, 0.01$ (left to right). The rightmost panel displays the two denoised mechanisms (aGM-JS and aGM-TH) in the high-dimensional case.}\label{fig:exp-hist}
\end{center}
\end{figure*}
\begin{figure*}
\captionsetup[subfigure]{justification=centering}
\begin{subfigure}[t]{0.245\textwidth}
\includegraphics[width=\textwidth]{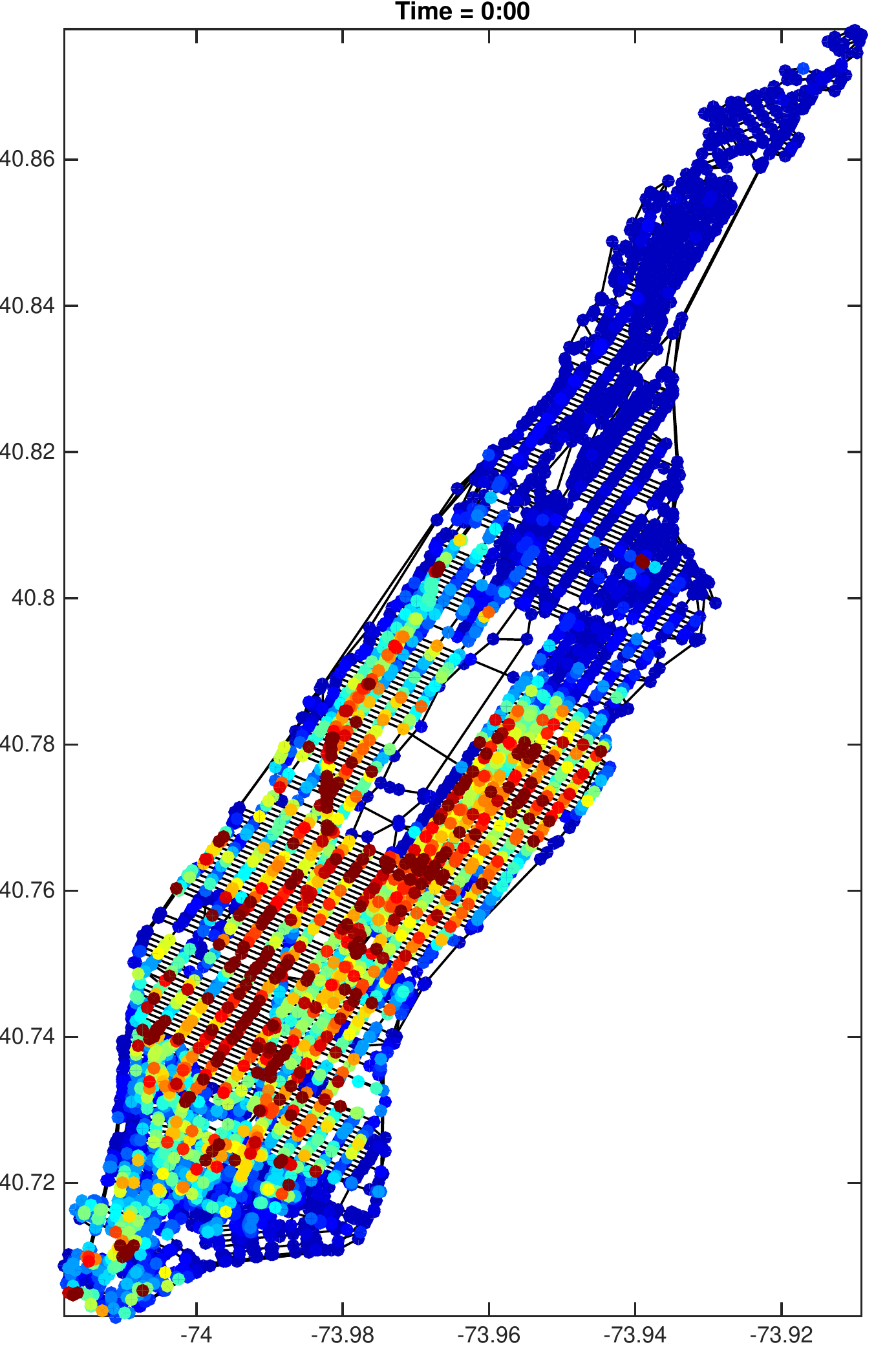}
\caption*{Ground truth}
\end{subfigure}
\begin{subfigure}[t]{0.245\textwidth}
\includegraphics[width=\textwidth]{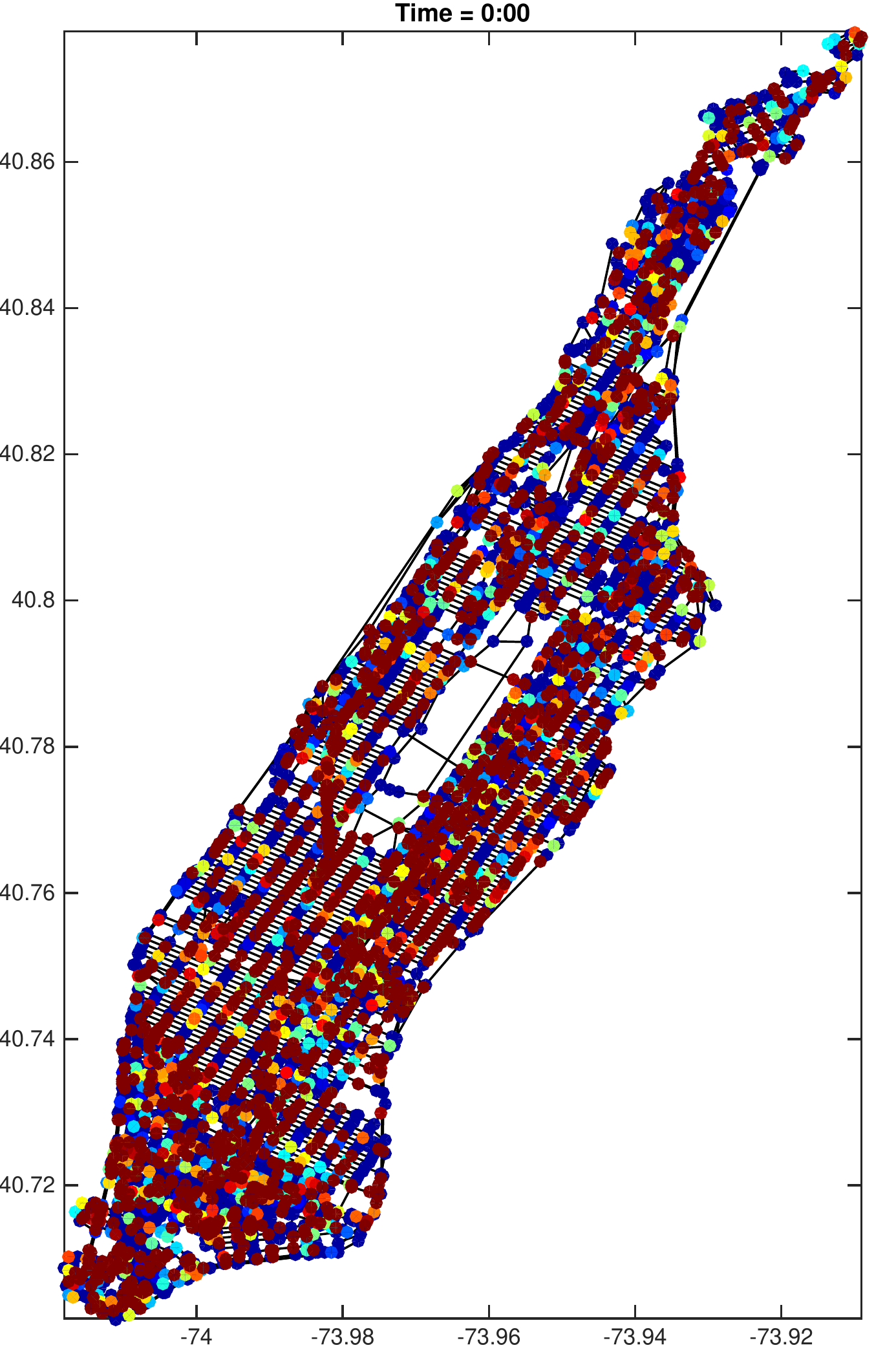}
\caption*{Raw aGM\\ RMSE = 40.29}
\end{subfigure}
\begin{subfigure}[t]{0.245\textwidth}
\includegraphics[width=\textwidth]{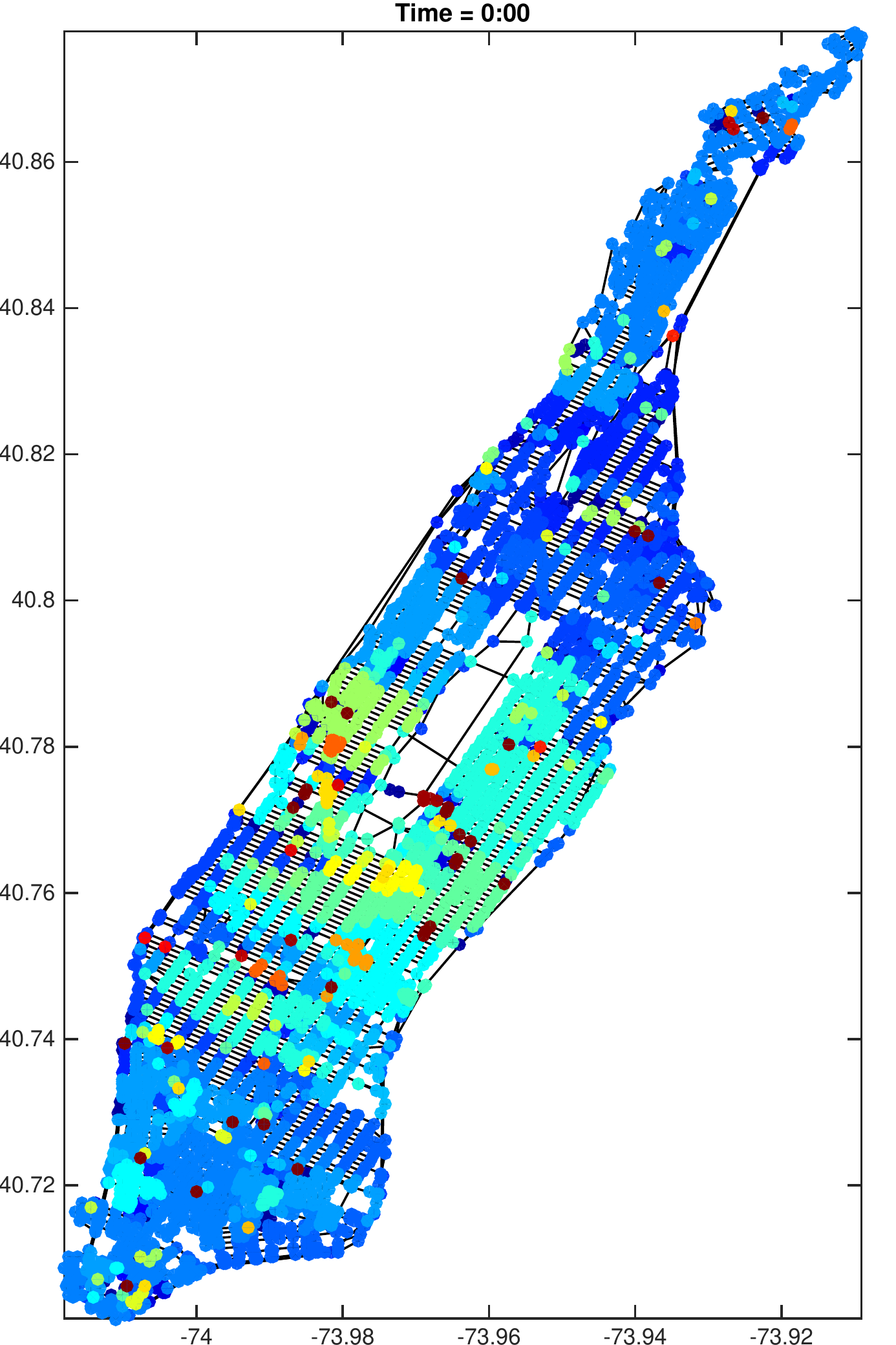}
\caption*{Wavelet\\ RMSE = 10.22}
\end{subfigure}
\begin{subfigure}[t]{0.245\textwidth}
\includegraphics[width=\textwidth]{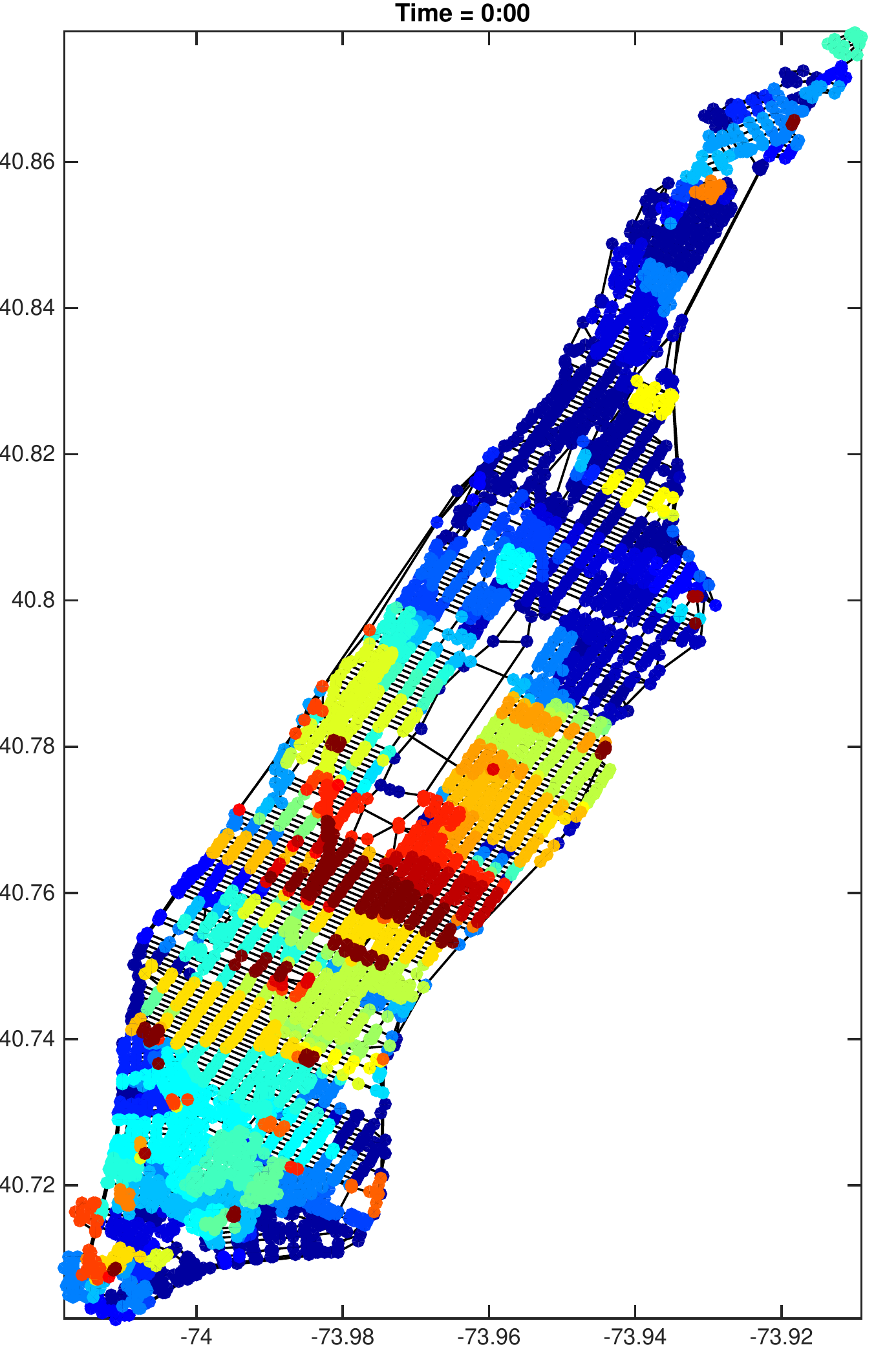}
\caption*{Trend filtering\\ RMSE = 9.46}
\end{subfigure}
\caption{Illustration of the denoising in differentially private release of NYC taxi density during 12:00 - 13:00 pm Sept 24, 2014. Comparing to the figure in the midnight of Figure~\ref{fig:illus_nyc_taxi}, the figures look structurally different. More activities center around the midtown and upper west sides.}\label{fig:illus_nyc_taxi_additional}
\end{figure*}
\begin{figure*}
\captionsetup[subfigure]{justification=centering}
	\begin{subfigure}[b]{0.245\textwidth}
		\includegraphics[width=\textwidth]{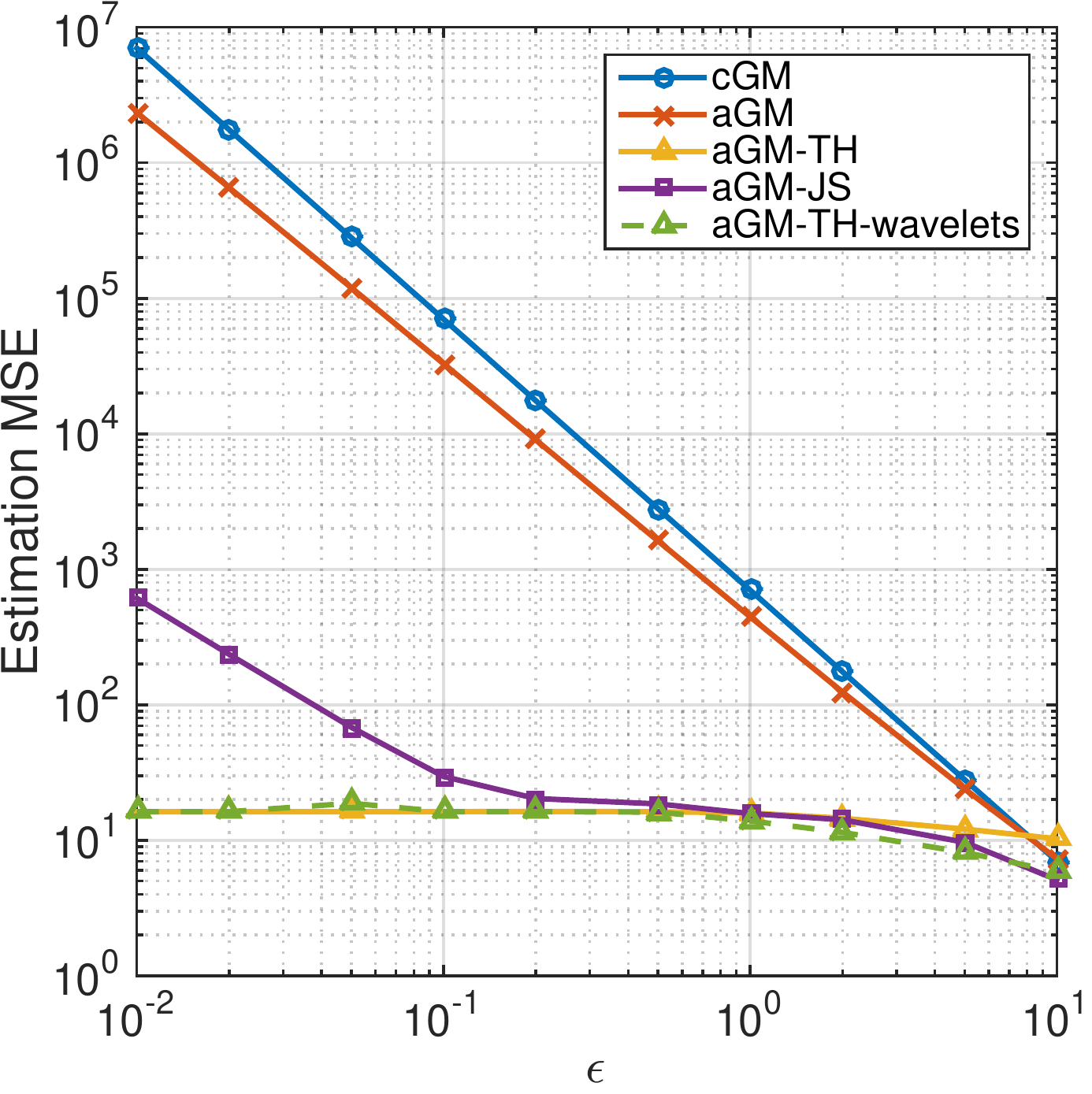}
		\caption*{8:00 am}
	\end{subfigure}
	\begin{subfigure}[b]{0.245\textwidth}
		\includegraphics[width=\textwidth]{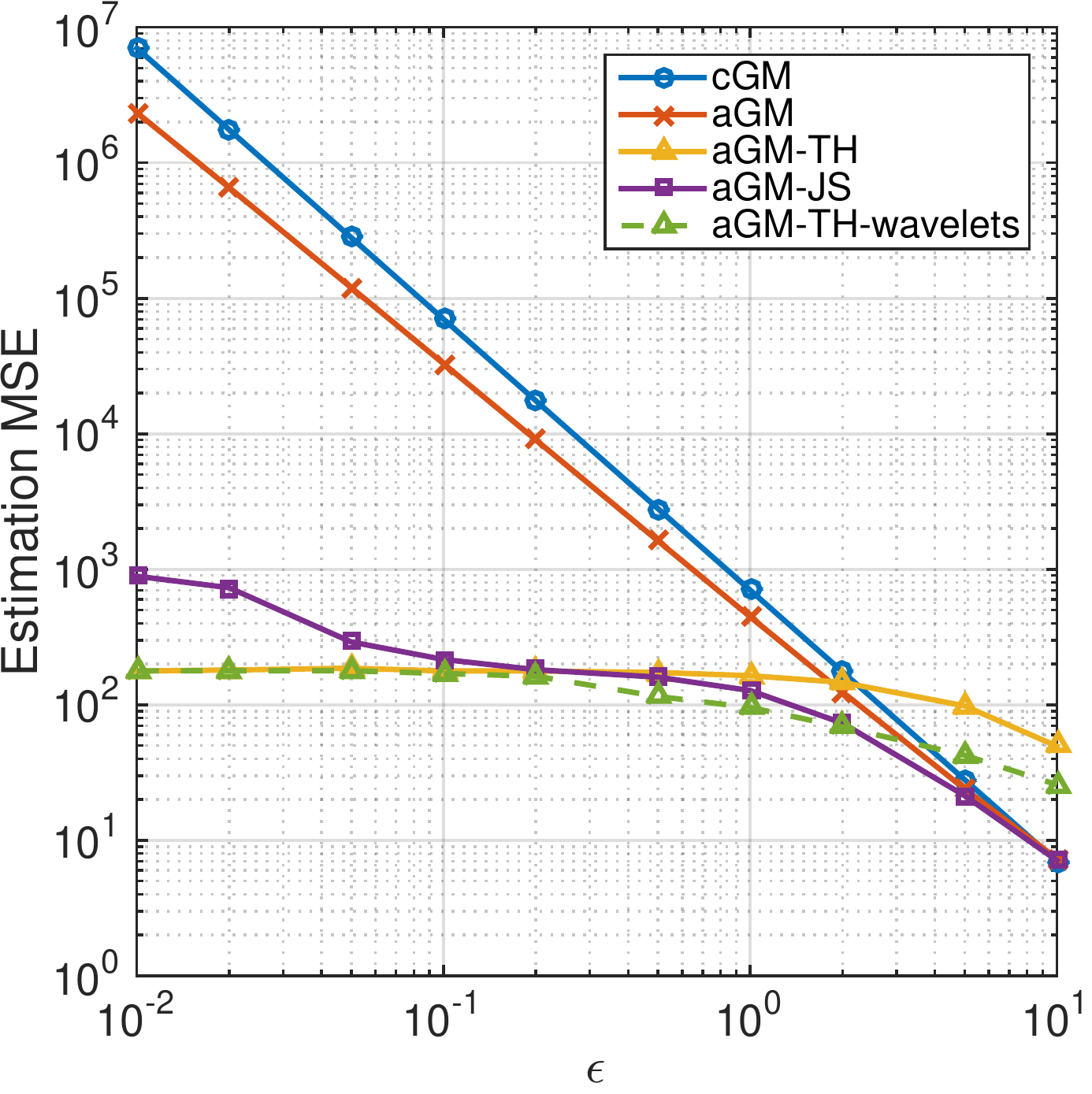}
		\caption*{12:00 am}
	\end{subfigure}
	\begin{subfigure}[b]{0.245\textwidth}
		\includegraphics[width=\textwidth]{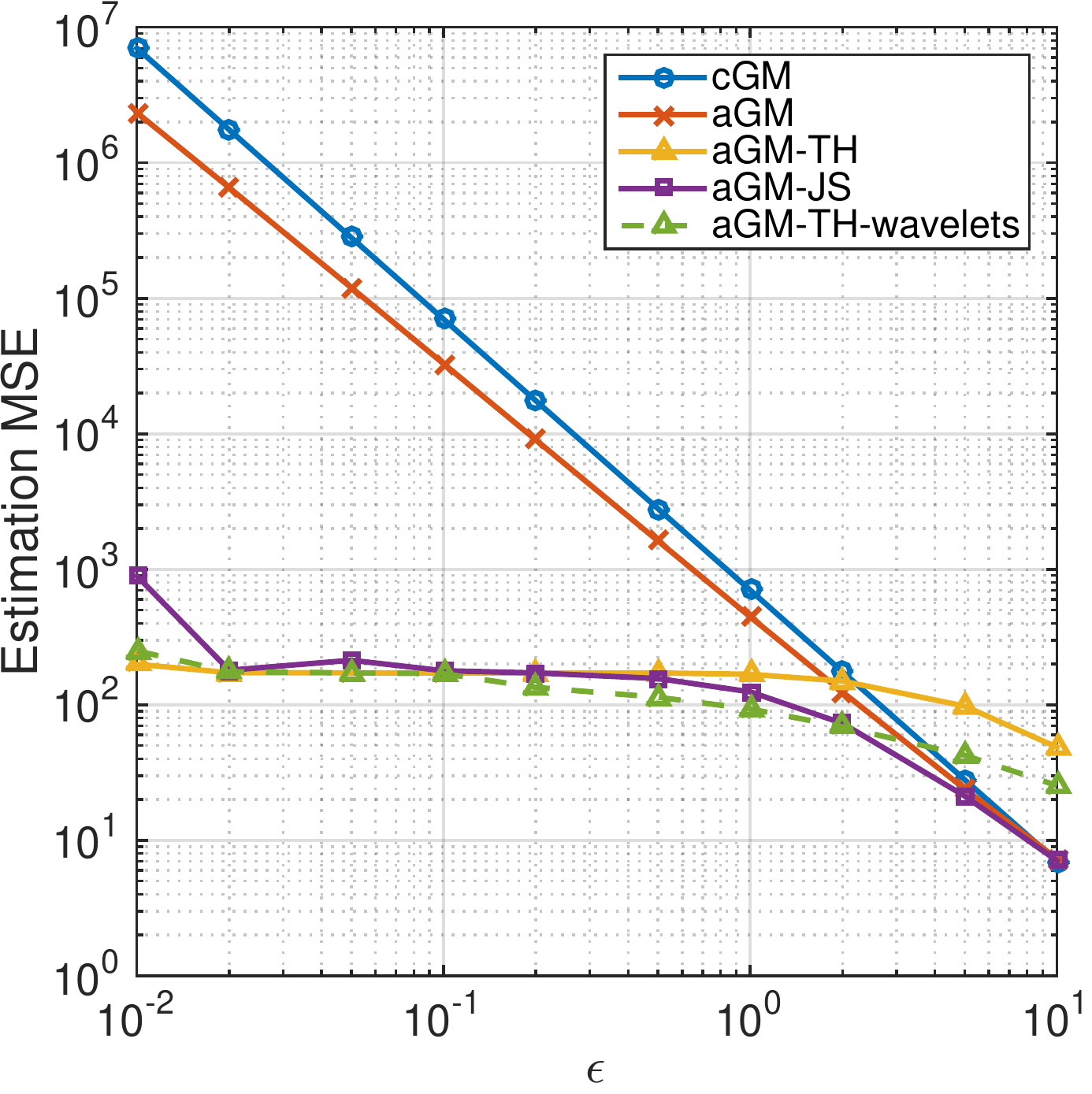}
		\caption*{16:00 am}
	\end{subfigure}
	\begin{subfigure}[b]{0.245\textwidth}
		\includegraphics[width=\textwidth]{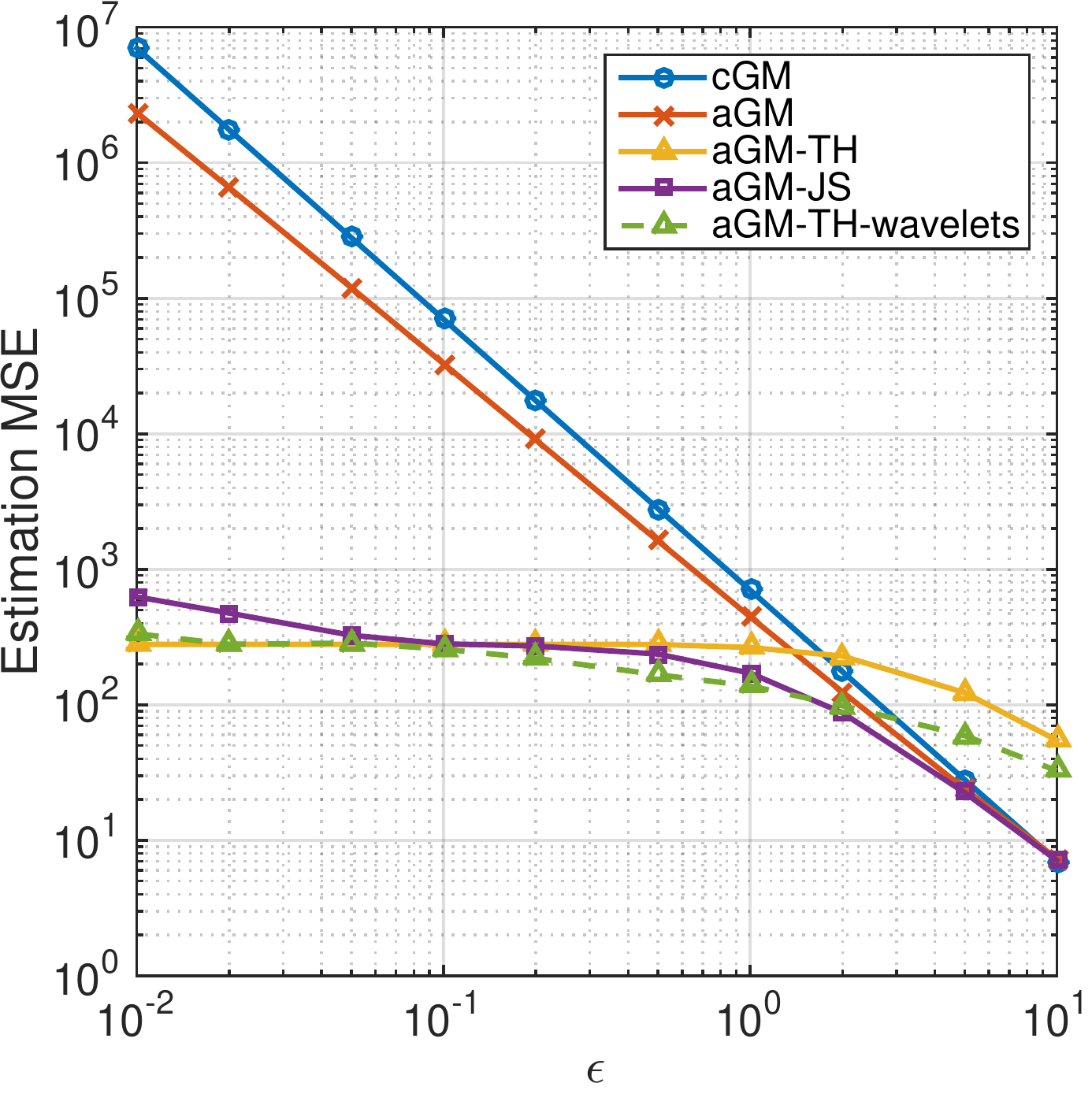}
		\caption*{20:00 am}
	\end{subfigure}
	\caption{Experiments for releasing NYC taxi heat maps. The plots compare the MSE of the reeased heat map as a function of the privacy loss parameter $\varepsilon$. We take $\Delta = 5$ and $\delta = 10^{-6}$ for all experiments. The wavelet basis is generated using \citet{sharpnack2013detecting} and the soft-thresholding's hyperparameter is chosen as $\sigma\sqrt{2\log d}$. }\label{fig:mse_nyc_taxi}	
\end{figure*}

Here we present additional experimental results. Figure~\ref{fig:exp-synth-extra} provides more plots for the setups explored in Sections~\ref{sec:expaGM} and~\ref{sec:expmean}. The next two sections present further experiments on a sparse histogram denoising task and on the New York City taxi dataset.

\subsection{Denoising for Histogram Release}\label{sec:exphist}
We evaluate the accuracy of our new Gaussian perturbation mechanisms on a second task involving private histogram release. In this problem the dataset $x = (x_1, \ldots, x_n)$ contains elements $x_i \in [d]$ from a finite set with $d$. The deterministic functionality is the empirical histogram $y = f(x) \in \R^{d}$ where $y_j = (1/n) \sum_{i=1}^n \mathbb{I}[x_i = j]$. In this case the global $L_2$ sensitivity $\Delta_2 = \sqrt{2}/n$ and a global $L_1$ sensitivity $\Delta_1 = 2/n$ (with respect to replacing one individual in a dataset by another arbitrary individual).

For this task, each dataset is sampled from a multinomial distribution with parameters sampled from a symmetric Dirichlet with $\alpha = 1/d$. The parameters are resampled for each individual experiment. The choice of $\alpha$ guarantees that the resulting histograms are highly sparse \cite{DBLP:journals/corr/abs-1301-4917}. Our setup follows the same structure as the one for the experiments from previous section. The results are presented in Figure~\ref{fig:exp-hist}. We observe that in this problem the Laplace mechanism is better than the classical Gaussian mechanism, and in the setting $\varepsilon = 1$ it is even better than the analytic Gaussian mechanism, with and without denoising. However, as we decrease $\varepsilon$ the utility of the analytic Gaussian mechanism becomes better than that of the Laplace mechanism, and denoising provides a significant advantage over mechanisms without denoising. Finally, we note that due to the sparsity of the underlying datapoint, denoising via soft thresholding provides better utility in this case than denoising via shrinking.

\subsection{New York City Taxi Heat Maps}

Here we present a second qualitative experiment with the New York City taxi dataset. The difference with the previous experiment is that we use data for a different time of the same day, leading to a different structure in the activities around the city; see Figure~\ref{fig:illus_nyc_taxi}. This illustrates that the selected denoising methods are adaptive to the structure of the underlying data.

Furthermore, Figure~\ref{fig:mse_nyc_taxi} presents quantitative results where we compare the mean square error (MSE) of cGM, aGM as well as the aforementioned denoising techniques. As we can see, on the real datasets, aGM always improves over cGM by a constant factor and denoising techniques are able to leverage bias-variance  trade-off and improve the recovery in MSE further. The benefits of denoising range from orders of magnitude (in the case when $\varepsilon$ is tiny) to a small constant factor (when $\varepsilon$ is moderate). In the low-privacy regime (e.g., $\varepsilon >5$), soft-thresholding performs a little worse than not using it at all. This is the expected cost of adaptivity and it does appear in its error bound.

\end{document}